\documentclass[11pt]{article}

\usepackage[T1]{fontenc}
\usepackage[preprint]{nexais}
\usepackage{newtxtext,newtxmath}
\usepackage[a4paper,margin=26mm]{geometry}
\usepackage[protrusion=true,expansion=false]{microtype}
\usepackage[round,authoryear]{natbib}
\usepackage[flushleft]{threeparttable}
\usepackage[font=small,labelfont=bf]{caption}
\usepackage[shortlabels]{enumitem}
\usepackage{flafter}
\usepackage{needspace}
\usepackage{xpatch}
\usepackage[breaklinks=true]{hyperref}
\nexaishypersetup
\usepackage[nameinlink,noabbrev]{cleveref}

\usepackage{mathtools}
\usepackage{amsthm}
\usepackage{amssymb}
\usepackage{graphicx}
\usepackage{booktabs}
\usepackage{colortbl}
\usepackage{multirow}
\usepackage{subcaption}
\usepackage{wrapfig}
\usepackage{placeins}
\usepackage[bottom,belowfloats]{footmisc}
\usepackage{xurl}
\usepackage{algorithm}
\usepackage{algpseudocode}

\makeatletter
\@ifpackageloaded{hyperref}{}{\usepackage{hyperref}}
\@ifpackageloaded{cleveref}{}{\usepackage[nameinlink,noabbrev]{cleveref}}

\@ifundefined{theorem}{%
  \theoremstyle{plain}
  \newtheorem{theorem}{Theorem}

  \theoremstyle{definition}
  
  \theoremstyle{remark}
  \newtheorem{remark}{Remark}
}{}
\makeatother

\crefname{theorem}{Theorem}{Theorems}
\crefname{lemma}{Lemma}{Lemmas}
\crefname{proposition}{Proposition}{Propositions}
\crefname{corollary}{Corollary}{Corollaries}
\crefname{definition}{Definition}{Definitions}
\crefname{conjecture}{Conjecture}{Conjectures}
\crefname{fact}{Fact}{Facts}
\crefname{claim}{Claim}{Claims}
\crefname{assumption}{Assumption}{Assumptions}
\crefname{example}{Example}{Examples}
\crefname{remark}{Remark}{Remarks}

\providecommand{\E}{\mathbb{E}}
\renewcommand{\P}{\mathbb{P}}
\providecommand{\R}{\mathbb{R}}
\providecommand{\argmax}{\operatorname*{arg\,max}}
\providecommand{\eps}{\epsilon}
\providecommand{\mathbbm}[1]{\mathbf{#1}}
\providecommand{\mask}{\mathbf{m}}
\providecommand{\mc}[1]{\mathcal{#1}}
\providecommand{\rva}{\mathbf{a}}
\providecommand{\rvc}{\mathbf{c}}
\providecommand{\rvs}{\mathbf{s}}
\providecommand{\rvx}{\mathbf{x}}
\providecommand{\bsl}{\backslash}
\providecommand{\brac}[1]{[ #1 ]}
\providecommand{\bracBig}[1]{\Big[ #1 \Big]}
\providecommand{\parenBig}[1]{\Big( #1 \Big)}
\providecommand{\set}[1]{\{ #1 \}}
\providecommand{\setBig}[1]{\Big\{ #1 \Big\}}
\providecommand{\dn}{\frac{1}{n}}
\providecommand{\sumn}{\sum_{i=1}^{n}}

\definecolor{morandiblue}{HTML}{D6EAF8}

\newcommand{\DFlowTitle}{dFlowGRPO: Rate-Aware Policy Optimization for Discrete Flow Models}
\newcommand{\DFlowKeywords}{discrete flow models, reinforcement learning, multimodal generation}

\setlist{nosep,leftmargin=1.6em}

\newcommand{\DFlowNexAISLayout}{}

\begin{document}

\title{\DFlowTitle}

\xpatchcmd{\affiliation}{\allowbreak\hspace{1.0em}}{\par}{}{%
  \PackageError{dflowgrpo}{Unable to adjust affiliation line breaks}{}%
}
\author[1,*]{Zhengyan Wan}
\author[2]{Yidong Ouyang}
\author[1]{Panwen Hu}
\author[1,3]{Qiang Sun}
\affiliation[1]{Mohamed bin Zayed University of Artificial Intelligence}
\affiliation[2]{University of California, Los Angeles}
\affiliation[3]{University of Toronto}
\contribution[*]{Work done as a visiting student at MBZUAI.}

\paperabstract{Discrete flow models (DFMs) are a class of flexible generative models for
generating discrete data, and diffusion large language models (dLLMs) can be
viewed as a special case with a specific choice of a mixture path and a masked
source distribution. While several recent works have explored reinforcement
learning for dLLMs, its application to more general discrete flow models
remains underexplored. In this work, we present discrete Flow-GRPO
(dFlowGRPO), a unified reinforcement learning framework for discrete flow
models that supports a broad family of probability paths and non-masked source
distributions. We derive the full trajectory probability for DFMs and
formulate the denoising process as a Markov decision process, enabling dFlowGRPO to
incorporate information from both the associated conditional transition rates
and the posterior model during reinforcement learning. We apply dFlowGRPO to
FUDOKI, a recent multimodal discrete flow model, and evaluate it on both image
generation and multimodal understanding tasks. Empirical results show that
dFlowGRPO outperforms existing GRPO-type methods for dLLMs on text-to-image
generation tasks and achieves performance competitive with continuous
flow-based models trained using Flow-GRPO, while also demonstrating strong
capabilities on understanding tasks.

}
\metadata[GitHub]{\url{https://github.com/WanZhengyan/dFlowGRPO}}
\date{September 30, 2026}
\pdfsubject{Reinforcement learning for discrete flow models}
\pdfkeywords{\DFlowKeywords}

\maketitle

\section{Introduction}
Diffusion large language models (dLLMs) \citep{nie2024scaling,nie2025large,ye2025dream} are a popular and powerful paradigm in text generation, speeding up the inference process by unmasking tokens in parallel, as an alternative to autoregressive models. Discrete flow models (DFMs) \citep{campbell2024generative,gat2024discrete,shaul2024flow} provide a flexible matching-based framework to learn the denoising process from a source distribution to the data distribution, by marginalizing the conditional transition rate similar to continuous flow matching \citep{albergo2022building,liu2022flow,lipman2022flow}, offering a larger design space than dLLMs. Specifically, dLLMs are a special case of DFMs under a specific choice of a mixture probability path and a conditional transition rate.

To improve the performance of dLLMs, several works on reinforcement learning for dLLMs have recently emerged. Since the log-likelihood of dLLMs cannot be decomposed naturally like that of autoregressive models, the existing works on RL for dLLMs mainly focus on the estimation of the log-likelihood of dLLMs, such as mean-field approximation \citep{zhao2025d1} and ELBO-based approximation \citep{zhu2025llada,ou2025principled}. Despite the success of these RL methods for dLLMs, no existing work focuses on RL for DFMs and takes general probability paths and conditional rates into account. In this work, we propose dFlowGRPO, a unified rate-aware policy optimization framework for DFMs by formulating the denoising process as a Markov decision process, parallel to Flow-GRPO \citep{liu2025flow}. By deriving the probability of the full trajectory, we find that the transition probability ratio of the denoising process of DFMs can be written as a product of token-wise expected posterior ratios weighted by a rate-dependent term, which can be estimated efficiently using MC samples without additional forward passes. Notably, the proposed GRPO objective incorporates information from both the conditional transition rate and the posterior model, providing an effective signal for policy optimization. See \cref{fig:overview} for an overview of dFlowGRPO.

Empirically, we evaluate the unified dFlowGRPO framework on both image generation and multimodal understanding tasks, using FUDOKI \citep[a multimodal DFM, ][]{wang2025fudoki} as our base model. For text-to-image generation, we use human preference PickScore \citep[model-based reward,][]{kirstain2023pick} and GenEval \citep[verifiable reward,][]{ghosh2023geneval} as our reward models for online RL training. Without classifier-free guidance, dFlowGRPO improves FUDOKI from 20.87 to 23.00 on PickScore and from 61\% to 93\% on GenEval without evident reward hacking across the reported auxiliary metrics, demonstrating the effectiveness and superior performance of dFlowGRPO on image generation tasks (see \cref{fig:overview} and \cref{tab:geneval}). For multimodal understanding, we train the base model on a multimodal dataset, ScienceQA \citep{lu2022learn}, with verifiable correctness reward (derived by exact matching and LLM extraction). dFlowGRPO boosts the understanding accuracy from 75.1\% to 81.1\% on ScienceQA (with LLM extraction), with limited evidence of reward hacking across the reported auxiliary benchmarks (see \cref{fig:understanding}). Additionally, we compare dFlowGRPO against other RL methods on FUDOKI, including online DPO \citep{guo2024direct}, diffu-GRPO and its variant \citep{zhao2025d1}. Experimental results show that dFlowGRPO consistently outperforms these approaches.

\begin{figure}
    \centering
    \includegraphics[width=0.95\linewidth]{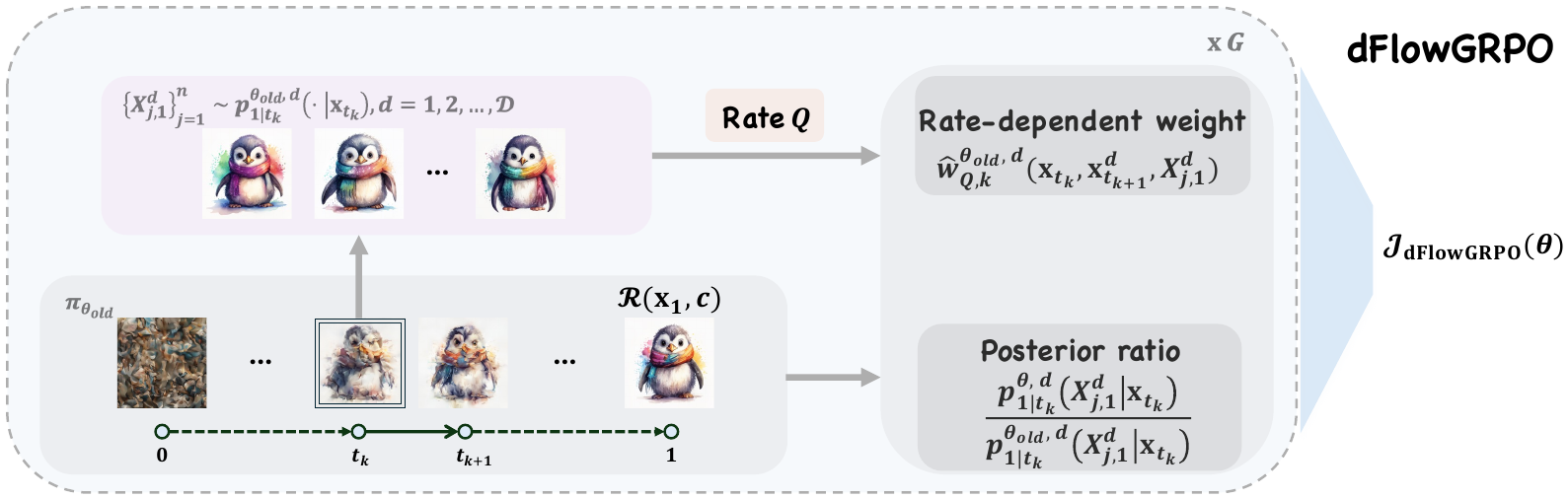}
    \includegraphics[width=0.5\linewidth]{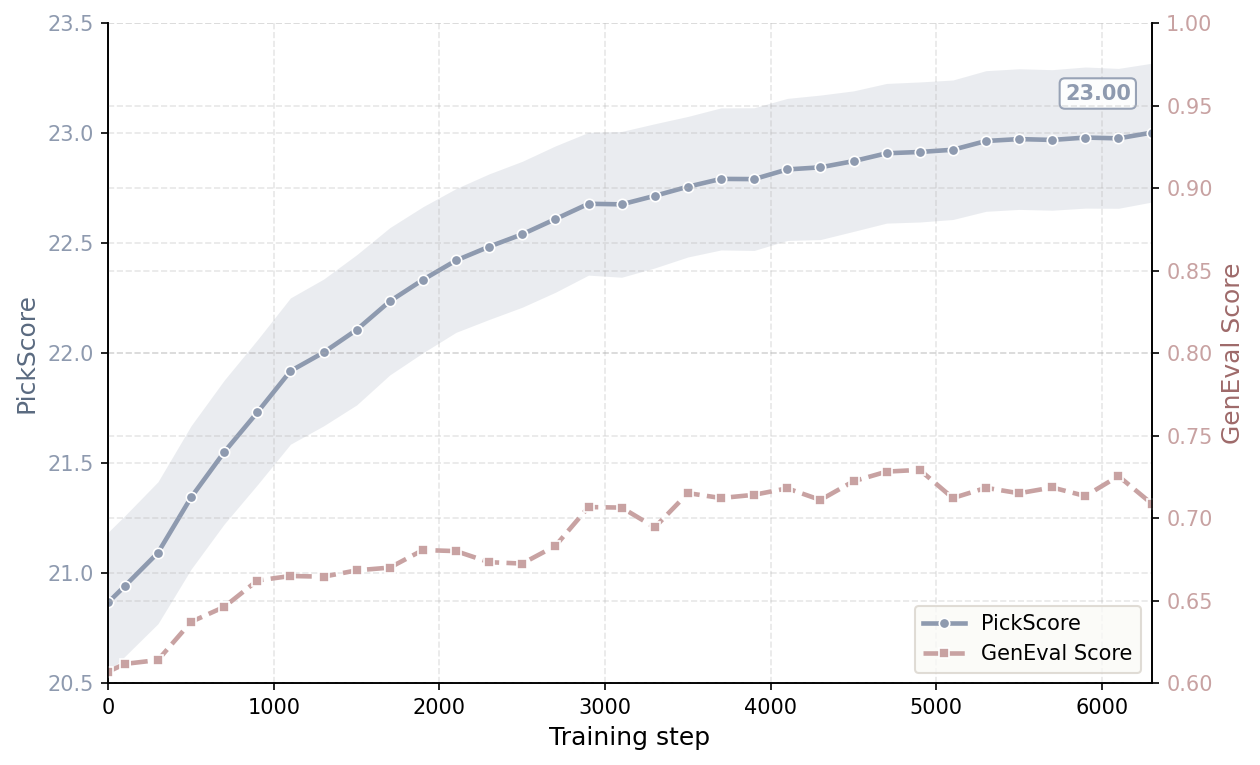}
    \includegraphics[width=0.35\linewidth]{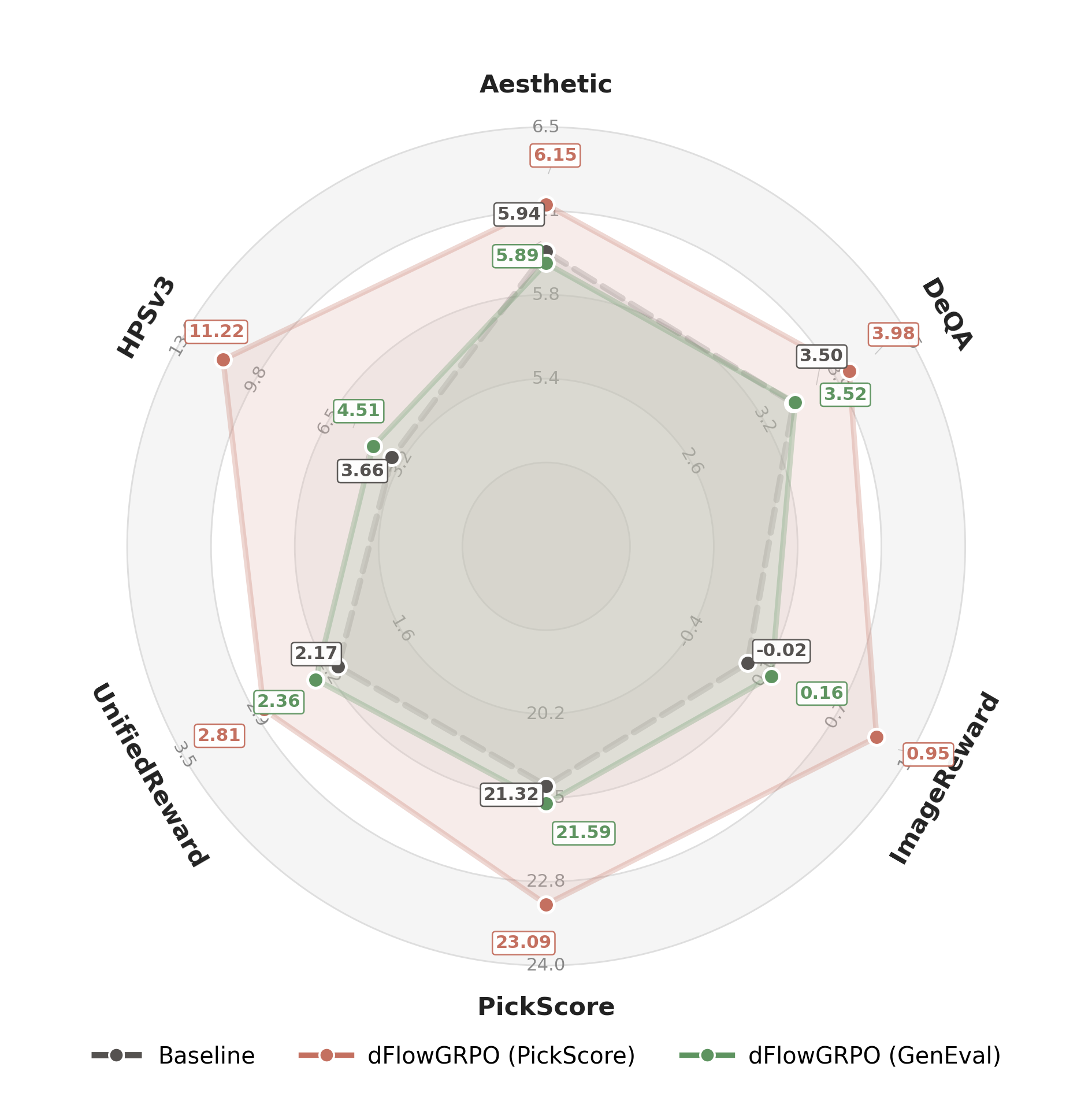}
    \caption{\small (Top) Overview of the proposed dFlowGRPO framework \eqref{eq:GRPO_DFM}. At each denoising step, we draw MC samples from the posterior model to compute the rate-dependent weight and posterior ratio, and then we calculate the rate-aware transition probability ratio by \eqref{eq:decoupled transition prob ratio}. (Bottom, Left) PickScore and GenEval score steadily increase during the training process of dFlowGRPO trained with PickScore reward without KL regularization. (Bottom, Right) The performance of dFlowGRPO trained with PickScore and GenEval reward without KL regularization, evaluated by other metrics on DrawBench prompts \citep{saharia2022photorealistic}. (These results use 8 NFEs for training and 20 NFEs for evaluation.)}
    \label{fig:overview}
\end{figure}

\paragraph{Contributions} We summarize our main contributions as follows.
\begin{enumerate}
    \item We propose dFlowGRPO, a unified rate-aware policy optimization framework for DFMs with general probability paths and conditional rates by utilizing information from both the conditional rate and the posterior model.
    \item We derive the transition probability ratio for the denoising trajectory, which can be written as a product of token-wise expected posterior ratios reweighted by a rate-dependent term. This term can be estimated using MC samples without additional forward passes.
    \item We empirically show the effectiveness of dFlowGRPO by applying it to a multimodal DFM. dFlowGRPO improves the reward of the base model significantly and consistently on text-to-image generation and multimodal understanding tasks.
\end{enumerate}

\section{Related Work}

\paragraph{Reinforcement learning for continuous diffusion and flow-based models}
Reinforcement learning methods, such as RLHF \citep{ouyang2022training} and RLVR \citep{lambert2024tulu}, have achieved significant successes in improving the capabilities of AR-based LLMs. Several works have boosted the performance of continuous diffusion and flow-based models via reinforcement learning by developing variants of existing policy gradient methods, such as PPO \cite{schulman2017proximal} and GRPO \cite{shao2024deepseekmath,guo2025deepseek}. For example, DPOK \citep{fan2023dpok} and DDPO \citep{black2023training} proposed policy gradient methods by formulating the denoising process of diffusion models as a Markov decision process, achieving better performance than reward-weighted methods \citep{lee2023aligning}. Flow-GRPO \cite{liu2025flow} and DanceGRPO \cite{xue2025dancegrpo} generalize GRPO to continuous flow-based models by converting ODEs to equivalent SDEs to derive transition probabilities. For DPO variants, \citet{Wallace_2024_CVPR} proposed Diffusion-DPO to improve text-to-image performance of diffusion models with a tractable DPO objective. \citet{liu2026improving} generalize DPO to continuous flow-based models for video generation.

\paragraph{Discrete diffusion and flow-based models}
As an alternative to AR-based models, discrete diffusion models \citep{austin2021structured,hoogeboom2021autoregressive,campbell2022continuous,sun2022score,lou2023discrete} aim to learn a denoising process, usually modeled as a continuous-time Markov chain (CTMC), which can transport a source distribution to the data distribution, parallel to the continuous counterparts \citep{ho2020denoising,song2020score,song2020denoising}. Diffusion large language models \citep{nie2024scaling,nie2025large} have shown promising results in language modeling, achieving performance comparable to that of autoregressive models. Instead of learning the time-reversal of the forward process in diffusion-based models, discrete flow-based models \citep{gat2024discrete,shaul2024flow,campbell2024generative} offer a flexible framework for learning a transition rate by marginalizing a pre-specified conditional transition rate (velocity), providing a larger design space than discrete diffusion models, similar to continuous flow matching \citep{liu2022flow,albergo2022building,lipman2022flow}.

\paragraph{Reinforcement learning for dLLMs}
The main challenge of RL for dLLMs is the log-likelihood estimation, which cannot be naturally decomposed as in autoregressive models. To tackle this issue, existing works have proposed various methods for estimating the log-likelihood, including mean-field approximation \citep{zhao2025d1} and ELBO approximation \citep{zhu2025llada,bie2025llada2,bie2026llada2,ma2026consolidating,ou2025principled}. Additionally, several variants and efficient strategies for ELBO-based methods have been developed, such as token-level ELBO-based methods \citep{yang2025mmada,gong2026diffucoder}, complementary masking \citep{li2026lavida,li2026lavidar1}, and semi-deterministic MC \citep{rojas2026improving}. \citet{tang2026wd} proposed an advantage-weighted log-likelihood maximization approach for RL. \citet{huang2026reinforcing} introduced diffusion chain-of-lateral-thought for dLLMs by optimizing the entire trajectory with a GRPO-based method. By formulating the denoising process of dLLMs as a Markov decision process, \citet{wang2025d2} and \citet{zhang2026dtrpo} derived the exact transition probability ratio for dLLMs and developed GRPO- and DPO-based methods, respectively. \citet{wang2026revolutionizing} accelerated RL training by aggregating neighboring steps and incorporated a value model to improve training stability.

\section{Preliminaries}
In this section, we briefly introduce discrete flow-based models and formulate the denoising process as a Markov decision process.

\subsection{Discrete Flow Models}\vspace{-3pt}
Consider a data distribution $q_1$ in a discrete space $\mathcal{S}^\mc{D}$, where $\mc{S}$ is the vocabulary and $\mc{D}$ is the sequence length. We call a stochastic process $(\rvx_t)_{t\in[0,1]}$ a continuous-time Markov chain (CTMC) on $\mc{S}^\mc{D}$ if it satisfies the Markov property and for any $x,z\in\mc{S}^\mc{D}$
\begin{align}\label{eq:CTMC}
    \P(\rvx_{t+h}=z|\rvx_{t}=x)=\delta_x(z)+Q_t(x,z)h+o(h),
\end{align}
where $(Q_t(x,z))_{x,z\in\mc{S}^\mc{D}}$ is the transition rate matrix satisfying the rate properties: $\sum_{z\in\mc{S}^\mc{D}}Q_t(x,z)=0$ and $Q_t(x,z)\mathbbm{1}(x\neq z)\ge 0$ for any $x,z\in\mc{S}^\mc{D}$. Here, $\mathbbm{1}(\cdot)$ is the indicator function. Denote $(q_t)_{t\in[0,1]}$ as the marginal probability distribution (path) corresponding to $(\rvx_t)_{t\in[0,1]}$. We say $Q_t$ can generate $q_t$ if it satisfies the Kolmogorov forward equation $\partial_t q_t(z)=\sum_{x\in\mc{S}^\mc{D}}q_t(x)Q_t(x,z)$ for $t\in[0,1]$.

Discrete flow models aim to learn a transition rate $Q_t$ that can transport a source distribution $q_0$ (e.g., a uniform distribution or a point mass at the mask token) to the data distribution $q_1$ by marginalizing a conditional transition rate $Q_t(x,z|x_1)$, parallel to continuous flow matching \cite{liu2022flow,albergo2022building}. Specifically, given a pre-defined conditional probability path $q_{t|1}(x|x_1)$ and a conditional transition rate $Q_t(x,z|x_1)$ that can generate $q_{t|1}$ conditional on $x_1$, the marginal transition rate $Q_t(x,z)=\sum_{x_1\in\mc{S}^\mc{D}}Q_t(x,z|x_1)q_{1|t}(x_1|x)$ can generate the marginal probability path $q_t$ \citep[see Proposition 3.1 in][]{campbell2024generative}, where $q_{1|t}$ is the posterior corresponding to $q_{t|1}$. For computational tractability, the conditional probability path and the associated transition rate are usually restricted to factorize across dimensions conditional on the terminal state $\rvx_1$ (that is, $q_{t|1}(x|x_1)=\prod_{d=1}^\mc{D}q^d_{t|1}(x^d|x_1^d)$ and $Q_t(x,z|x_1)=\sum_{d=1}^\mc{D}\delta_{x^{\bsl d}}(z^{\bsl d})Q_t^d(x^d,z^d|x_1^d)$), which implies that the marginal transition rate matrix is sparse:
\begin{align*}
   Q_t(x,z)=\sum_{d\in\brac{\mc{D}}}\delta_{x^{\bsl d}}(z^{\bsl d})Q_t^d(x,z^d),
\end{align*}
where $Q_t^d(x,z^d)=\sum_{x_1^d\in\mc{S}}Q_t^d(x^d,z^d|x_1^d)q^d_{1|t}(x_1^d|x)$. In particular, choosing the mixture path $q^d_{t|1}(x^d|x_1^d)=\kappa_t \delta_{x_1^d}(x^d)+(1-\kappa_t)\delta_\mask(x^d)$ with the conditional rate $Q_t^{d}(x^d,z^d|x_1^d)=\frac{\dot{\kappa}_t}{1-\kappa_t}(\delta_{x_1^d}(z^d)-\delta_{x^d}(z^d))$, we can recover the masked diffusion models \citep{shi2024simplified,ou2024your,sahoo2024simple,nie2025large}, where $\kappa_t$ is the time scheduler.

\paragraph{Training} In practice, we can parametrize the posterior $q_{1|t}$ using a neural network $p_{1|t}^{\theta,d}$. For general conditional probability paths, a generally applicable training objective is the following cross-entropy loss \citep{campbell2024generative,shaul2024flow,gat2024discrete,wang2025fudoki}

$$\mc{L}(\theta)=-\E_{\rvx_1\sim q_1(\cdot),\rvx_t\sim q_{t|1}(\cdot|\rvx_1)}\brac{\sum_{d\in\brac{\mc{D}}}\log p^{\theta,d}_{1|t}(\rvx_1^d|\rvx_t)}.$$

\paragraph{Sampling} To improve sampling efficiency for general probability paths and to tackle the ill-defined issue of \eqref{eq:CTMC} when the step size $h$ is not sufficiently small, one can use the following always-valid Euler sampler\footnote{Refer to the official flow matching PyTorch library: \url{https://github.com/facebookresearch/flow_matching}} at each denoising step \cite{shaul2024flow,wang2025fudoki}:
\begin{enumerate}
    \item Given $\rvx_t$, for each $d\in\brac{\mc{D}}$, sample $X_1^d\sim p_{1|t}^{\theta,d}(\cdot|\rvx_t)$ in parallel, where $\brac{\mc{D}}=\set{1,2,\dots,\mc{D}}$;
    \item For each $d\in\brac{\mc{D}}$, set $\rvx_{t+h}^d=\rvx_t^d$ with probability $\exp(-h\lambda_t^d(\rvx_t^d,X_1^d))$ or set $\rvx_{t+h}^d=z^d\neq \rvx_t^d$ with probability $\setBig{1-\exp(-h\lambda_t^d(\rvx_t^d,X_1^d))}\frac{Q_t^d(\rvx_t^d,z^d|X_1^d)}{\lambda_t^d(\rvx_t^d,X_1^d)}$, where $\lambda_t^d(\rvx_t^d,X_1^d)=\sum_{z^d\neq \rvx_t^d}Q^d_{t}(\rvx_t^d,z^d|X_1^d)$.
\end{enumerate}
The details of the Euler solver can be found in \cref{alg:Euler sampler in general}.

\subsection{Denoising as a Markov Decision Process}\vspace{-3pt}

Consider a time discretization scheme $0=t_0<t_1<\cdots<t_K=1$. Let $p(\rvx_{t_{k+1}}|\rvx_{t_{k}},\rvc)$ be the transition probability at the $k$-th timestep given a prompt (or context) $\rvc$. Following DDPO \citep{black2023training} and Flow-GRPO \citep{liu2025flow}, we can formulate the denoising process as a Markov decision process defined by the tuple $(\mathbf{S}, \mc{A}, P, r, \rho_0)$, where $\mathbf{S}$ is a set of states, $\mc{A}$ is a set of actions, $P:\mathbf{S}\times\mc{A}\times\mathbf{S}\to \R$ is the transition probability, $r:\mathbf{S}\times\mc{A}\to\R$ is the reward function and $\rho_0:\mathbf{S}\to \R$ is the distribution of the initial state $s_0$ (considering a constant discount factor for simplicity). At the $k$-th step with a state $\rvs_k=(\rvc,k,\rvx_{t_k})$, an agent takes a stochastic action $\rva_k=\rvx_{t_{k+1}}$ with probability $\pi(\rva_{{k}}|\rvs_{k})=p(\rvx_{t_{k+1}}|\rvx_{t_{k}},\rvc)$ and receives a reward $R(\rvs_k,\rva_k)$, moving to the next state $\rvs_{k+1}$ with deterministic transition $P(\rvs_{k+1}|\rvs_{{k}},\rva_{{k}})=(\delta_{\rvc},\delta_{k+1},\delta_{\rvx_{t_{k+1}}})$. In continuous diffusion or flow-based models, it is common to take $\rho_0(\rvs_0)=(p_c(\rvc),\delta_0,\mc{N}(\mathbf{0},\mathbf{I}))$ and assign the reward of the entire denoising trajectory to the terminal state; that is, $R(\rvs_k,\rva_k)=\mc{R}(\rvx_{1},\rvc)$ if $k=K-1$, and $0$ otherwise. The goal of RL is to maximize the expected return $\E_{\rvc\sim p_c,\pi_\theta(\cdot|\rvc)}\brac{\sum_{k=0}^{K-1}R(\rvs_k,\rva_k)}$, which has the gradient $\E_{\rvc\sim p_c,\pi_\theta(\cdot|\rvc)}\brac{\sum_{k=0}^{K-1}\nabla_\theta\log p_\theta(\rvx_{t_{k+1}}|\rvx_{t_{k}},\rvc)\mc{R}(\rvx_1,\rvc)} $.

\section{Discrete Flow-GRPO}

In this section, we introduce Discrete Flow-GRPO (dFlowGRPO), a unified reinforcement learning framework for discrete flow models based on the GRPO framework \cite{shao2024deepseekmath,guo2025deepseek}.

\subsection{GRPO for Discrete Flow Models}\vspace{-3pt}

Given a prompt $\rvc\sim p_c$ and the number of denoising steps $K$, we can generate $G$ final samples $\set{\rvx_1^{(i)}}_{i\in\brac{G}}$ using the Euler solver discussed in the previous section. Let $\set{(\rvx^{(i)})}_{i\in\brac{G}}\overset{\triangle}{=}\set{(\rvx^{(i)}_{t_0},\rvx^{(i)}_{t_1},\dots,\rvx^{(i)}_{t_K})}_{i\in\brac{G}}$ be the corresponding denoising trajectories. Since the denoising process can be formulated as a Markov decision process and the full trajectory probability can be decoupled $(p(\rvx^{(i)})=p(\rvx_0^{(i)})\prod_{k=0}^{K-1}p(\rvx_{t_{k+1}}^{(i)}|\rvx_{t_{k}}^{(i)},\rvc))$, following Flow-GRPO \citep{liu2025flow}, we can use the following training objective to optimize the policy model:
{\small\begin{equation}\label{eq:GRPO_DFM}
   \begin{aligned}
    &~\mc{J}_{\text{dFlowGRPO}}(\theta)=\E_{\rvc\sim p_c,\set{\rvx^{(i)}}_{i=1}^G\sim\pi_{\theta_{old}}(\cdot|\rvc)}\\
    &~\quad\setBig{\frac{1}{G}\sum_{i=1}^G\frac{1}{K}\sum_{k=0}^{K-1}\parenBig{\min(\brac{r_k^{(i)}(\theta)}^{\frac{1}{\mc{D}}}\hat{A}_k^{(i)},\text{clip}(\brac{r_k^{(i)}(\theta)}^{\frac{1}{\mc{D}}},1-\eps,1+\eps)\hat{A}_k^{(i)})-\beta D_{KL}(\pi_\theta\|\pi_{ref})}},
\end{aligned} 
\end{equation}}
where $\eps$ is the clipping parameter, $r_k^{(i)}(\theta)$ is the step-level transition probability ratio between the current policy and the old policy and the estimated advantage $\hat{A}_k^{(i)}$ can be computed using the normalized reward in the group, which is independent of $k$:
\begin{align*}
r_k^{(i)}(\theta)=&~\frac{p_\theta(\rvx_{t_{k+1}}^{(i)}|\rvx_{t_{k}}^{(i)},\rvc)}{p_{\theta_{old}}(\rvx_{t_{k+1}}^{(i)}|\rvx_{t_{k}}^{(i)},\rvc)}; \quad \hat{A}_k^{(i)}=\frac{\mc{R}(\rvx_{1}^{(i)},\rvc)-\text{mean}(\set{\mc{R}(\rvx_{1}^{(i)},\rvc)}_{i\in\brac{G}})}{\text{std}(\set{\mc{R}(\rvx_{1}^{(i)},\rvc)}_{i\in\brac{G}})}.
\end{align*}

\begin{remark}
    Since the transition probability ratio can be written as a product of token-level probability ratios (see \eqref{eq:decoupled transition prob ratio}), we use the geometric mean of the token-level probability ratios for training stability in \eqref{eq:GRPO_DFM}, which is similar to GSPO \citep{zheng2025group}, GMPO \citep{zhao2025geometric} and ESPO \citep{ou2025principled}. We summarize dFlowGRPO in \cref{alg:dFlowGRPO} and derive its gradient in \cref{sec:appendix grad}. Similarly, we use the following estimator for KL regularization to improve training stability:
    \begin{align*}
        D_{KL}(\pi_\theta\|\pi_{ref})\approx&~ \brac{r_k^{(i)}(\theta)}^{\frac{1}{\mc{D}}}-\log \brac{r_k^{(i)}(\theta)}^{\frac{1}{\mc{D}}}-1\\
        \leq&~ \frac{1}{\mc{D}}\sum_{d\in\brac{\mc{D}}}\setBig{\frac{p^d_\theta(\rvx_{t_{k+1}}^d|\rvx_{t_k})}{p^d_{\theta_{old}}(\rvx_{t_{k+1}}^d|\rvx_{t_k})}-\log\frac{p^d_\theta(\rvx_{t_{k+1}}^d|\rvx_{t_k})}{p^d_{\theta_{old}}(\rvx_{t_{k+1}}^d|\rvx_{t_k})}-1},
    \end{align*}
    where the right-hand side is the token-level KL regularization similar to standard GRPO \citep{shao2024deepseekmath}.
\end{remark}

\subsection{Rate-Aware Transition Probability Ratio}\vspace{-3pt}
To compute the GRPO objective, we have to calculate the transition probability ratio for the Euler solver (\cref{alg:Euler sampler in general}). Following the discussion in \cite{wan2026corrected} and \cite{liang2025discrete}, in the time interval $[t_{k},t_{k+1}]$, the Euler sampler can be viewed as a CTMC with a time-homogeneous transition rate $Q^{\text{Euler}}_t(x,z)=\sum_{d=1}^\mc{D}\delta_{x^{\bsl d}}(z^{\bsl d})\delta_{\rvx_{t_k}^d}(x^d)\hat{Q}_{t_{k}}^d(\rvx_{t_k},z^d)$, where $\hat{Q}_{t_{k}}^d(\rvx_{t_k},z^d)=Q^d_{t_{k}}(\rvx_{t_k},z^d|X_1^d)$ is the estimator of the oracle rate using a single MC sample $X_1^d\sim p_{1|t}^d(\cdot|\rvx_{t_k})$. We first define the following unnormalized rate-dependent weight:
\begin{equation}\label{eq:unnormalized weight}
    \begin{aligned}
    &~\tilde{w}_{Q,k}^{d}(\rvx_{t_{k}},\rvx_{t_{k+1}}^d,X_1^d)=\mathbbm{1}(\rvx_{t_k}^d=\rvx_{t_{k+1}}^d)\parenBig{\exp\parenBig{-(t_{k+1}-t_k)\lambda_{t_k}^d(\rvx_{t_k}^d,X_1^d)}}\\
        &~\qquad \qquad  +\mathbbm{1}(\rvx_{t_k}^d\neq\rvx_{t_{k+1}}^d)\parenBig{\frac{Q^d_{t_k}(\rvx_{t_k}^d,\rvx_{t_{k+1}}^d|X_1^d)}{\lambda_{t_k}^d(\rvx_{t_k}^d,X_1^d)}\setBig{1-\exp\parenBig{-(t_{k+1}-t_k)\lambda_{t_k}^d(\rvx_{t_k}^d,X_1^d)}}},
\end{aligned}
\end{equation}
where $\lambda_{t_k}^d(\rvx_{t_k}^d,X_1^d)=\sum_{z^d\neq \rvx_{t_k}^d}Q^d_{t_k}(\rvx_{t_k}^d,z^d|X_1^d)$. We derive the transition probability ratio for the denoising trajectory of DFMs in the following theorem.

\begin{theorem}\label{thm:transition ratio}
    For sampling with the Euler solver (\cref{alg:Euler sampler in general}) and a posterior model $p_{1|t}^{\theta}$, the transition probability is $p_\theta(\rvx_{t_{k+1}}|\rvx_{t_{k}})=\prod_{d\in\brac{\mc{D}}}p_\theta^d(\rvx_{t_{k+1}}^{d}|\rvx_{t_{k}})$ (here, we omit the prompt $\rvc$ for simplicity), where
    \begin{align*}
p^d_\theta(\rvx_{t_{k+1}}^d|\rvx_{t_k})=\E_{X_1^d\sim p^{\theta,d}_{1|t_{k}}(\cdot|\rvx_{t_{k}})}\brac{\tilde{w}_{Q,k}^{d}(\rvx_{t_{k}},\rvx_{t_{k+1}}^d,X_1^d)}.
    \end{align*}
    Consequently, the transition probability ratio between the current policy and the old policy is
{\small    \begin{align}\label{eq:decoupled transition prob ratio}
        r_k(\theta)=\prod_{d\in\brac{\mc{D}}}\frac{p^d_\theta(\rvx_{t_{k+1}}^d|\rvx_{t_k})}{p^d_{\theta_{old}}(\rvx_{t_{k+1}}^d|\rvx_{t_k})}=\prod_{d\in\brac{\mc{D}}}\E_{X_1^d\sim p^{\theta_{old}}_{1|t_{k}}(\cdot|\rvx_{t_{k}})} \setBig{\underbrace{w_{Q,k}^{\theta_{old},d}(\rvx_{t_{k}},\rvx_{t_{k+1}}^d,X_1^d)}_{\text{rate-dependent weight}}\underbrace{\frac{p^{\theta}_{1|t_{k}}(X_1^d|\rvx_{t_{k}})}{p^{\theta_{old}}_{1|t_{k}}(X_1^d|\rvx_{t_{k}})}}_{\text{posterior ratio}}},
    \end{align}}
    where the rate-dependent weight (satisfying $\E_{X_1^d\sim p^{\theta}_{1|t_{k}}(\cdot|\rvx_{t_{k}})} \brac{w_{Q,k}^{\theta,d}(\rvx_{t_{k}},\rvx_{t_{k+1}}^d,X_1^d)}=1$) is defined by
    \begin{align}\label{eq:normalized weight}
       &~w_{Q,k}^{\theta,d}(\rvx_{t_{k}},\rvx_{t_{k+1}}^d,X_1^d)=\frac{\tilde{w}_{Q,k}^{d}(\rvx_{t_{k}},\rvx_{t_{k+1}}^d,X_1^d)}{p_{\theta}^d(\rvx_{t_{k+1}}^d|\rvx_{t_k})}. 
    \end{align}
\end{theorem}

\Cref{thm:transition ratio} indicates that the transition probability ratio in the denoising trajectory is a product of dimension-wise expected posterior ratios, each reweighted by a rate-dependent term. In particular, for the mixture path with masked source distribution $q_0^d(x)=\delta_\mask$ and conditional rate $Q_t^{d}(x^d,z^d|x_1^d)=\frac{\dot{\kappa}_t}{1-\kappa_t}(\delta_{x_1^d}(z^d)-\delta_{x^d}(z^d))$, the transition probability ratio is
\begin{align*}
    r_k^{(i)}(\theta)=\prod_{d: \rvx_{t_k}^d=\mask \text{ and } \rvx_{t_{k+1}}^d\neq\mask}\frac{p_{1|t}^{\theta,d}(\rvx_{t_{k+1}}^d|\rvx_{t_k})}{p_{1|t}^{\theta_{old},d}(\rvx_{t_{k+1}}^d|\rvx_{t_k})},
\end{align*}
which recovers the transition probability ratio derived in Theorem 3.2 of \cite{zhang2026dtrpo}. We defer the proof of \cref{thm:transition ratio} to \cref{sec:proof}.

\paragraph{Estimating rate-dependent weight}

For a general probability path \citep[e.g., metric-induced probability path introduced in][]{shaul2024flow,wang2025fudoki}, we can use the empirical version of the transition probability ratio to estimate $r_k(\theta)$ at time $t_k$ without additional forward passes; that is,
\begin{align}\label{eq:ratio estimator}
    \hat{r}_k(\theta)=\prod_{d\in\brac{\mc{D}}}\frac{\hat{p}^d_\theta(\rvx_{t_{k+1}}^d|\rvx_{t_k})}{\hat{p}^d_{\theta_{old}}(\rvx_{t_{k+1}}^d|\rvx_{t_k})}=\prod_{d\in\brac{\mc{D}}}\parenBig{\frac{1}{n}\sum_{j=1}^{n} \setBig{\hat{w}_{Q,k}^{\theta_{old},d}(\rvx_{t_{k}},\rvx_{t_{k+1}}^d,X_{j,1}^d)\frac{p^{\theta}_{1|t_{k}}(X_{j,1}^d|\rvx_{t_{k}})}{p^{\theta_{old}}_{1|t_{k}}(X_{j,1}^d|\rvx_{t_{k}})}}},
\end{align}
where $\set{X_{j,1}^d}_{j\in\brac{n}}$ are i.i.d. samples from $p^{\theta_{old},d}_{1|t_{k}}(\cdot|\rvx_{t_k})$ and
\begin{align}\label{eq:weight estimator}
    \hat{w}_{Q,k}^{\theta_{old},d}(\rvx_{t_{k}},\rvx_{t_{k+1}}^d,X_{j,1}^d)=\frac{\tilde{w}_{Q,k}^{d}(\rvx_{t_{k}},\rvx_{t_{k+1}}^d,X_{j,1}^d)}{\frac{1}{n}\sum_{l=1}^{n}\tilde{w}_{Q,k}^{d}(\rvx_{t_{k}},\rvx_{t_{k+1}}^d,X_{l,1}^d)}.
\end{align}
When we only use one single Monte Carlo sample (i.e., $n=1$), the weight estimator is equal to one, which implies that the transition probability ratio estimator reduces to the product of the token-wise posterior ratios. In this scenario, the estimation of the transition probability ratio is not only independent of the conditional rate but also unrelated to the next state $\rvx_{t_{k+1}}$, which has a non-negligible bias. Additionally, the ratio estimation with a single MC sample also leads to slow convergence during training; we will discuss this in the ablation study.

\section{Experiments}
In this section, we conduct comprehensive experiments on image generation and multimodal understanding tasks to evaluate the performance of the proposed dFlowGRPO framework.
Results on a second DFM are provided in \cref{sec:appendix_second_dfm}.
\subsection{Experimental Setup}\vspace{-3pt}
We choose FUDOKI \citep{wang2025fudoki}, a multimodal DFM, as our base model. The metric-induced probability path and the corresponding kinetic-optimal conditional rate \citep{shaul2024flow} for FUDOKI are defined by
\begin{align*}
    q_{t|1}^d(x^d|x_1^d)=&~\text{softmax}(-\beta_t\cdot \tilde{d}(x^d,x_1^d)),\\
    Q^d_t(x^d,z^d|x_1^d)=&~q^d_{t|1}(z^d|x_1^d)\dot{\beta}_t\brac{\tilde{d}(x^d,x_1^d)-\tilde{d}(z^d,x_1^d)},
\end{align*}
where $\beta_t=3\cdot(\frac{t}{1-t})^{0.9}$ and $\tilde{d}(\cdot,\cdot)$ is a metric calculated based on the embedding layers adopted by FUDOKI. For image generation, the sequence length for the response is set to 576 and the resolution of the generated images is $384\times 384$ in FUDOKI. For multimodal understanding, the sequence length for the text part is set to $500$. To accelerate training, we do not use classifier-free guidance, and the number of denoising steps is set to $8$ by default for all experiments. We adopt LoRA with rank $48$ and $\alpha=96$ to fine-tune FUDOKI. The implementation details can be found in \cref{sec:implementation details}.
\begin{table}[h!]
\centering
\caption{\small Evaluation of text-to-image generation ability on the GenEval benchmark. $^{\dagger}$The results of FUDOKI (with CFG) are presented in their original paper.}
\label{tab:geneval}
\vspace{3pt}
\resizebox{1\textwidth}{!}{%
\begin{tabular}{lccccccc}
\toprule
\textbf{Model} 
   & \textbf{Single Obj.} & \textbf{Two Obj.} & \textbf{Counting} & \textbf{Colors} & \textbf{Position} & \textbf{Attribute}  & \textbf{Overall $\uparrow$} \\
\midrule
\multicolumn{8}{l}{\textit{\textbf{Generation-Only}}} \\
SDXL~\citep{Podell2023SDXLIL}  & 0.98 & 0.74 & 0.39 & 0.85 & 0.15 & 0.23 & 0.55 \\
DALL-E 3~\citep{BetkerImprovingIG}  & 0.96 & 0.87 & 0.47 & 0.83 & 0.43 & 0.45 & 0.67 \\
SD3.5-L~\citep{Esser2024ScalingRF}  & 0.98 & 0.89 & 0.73 & 0.83 & 0.34 & 0.47 & 0.71 \\
FLUX.1-dev~\citep{flux2024}  & 0.98 & 0.93 & 0.75 & 0.93 & 0.68 & 0.65 & 0.82 \\
SD3.5-M~\citep{Esser2024ScalingRF}  & 0.98& 0.78& 0.50& 0.81& 0.24& 0.52 & 0.63 \\
\ \ w/ Flow-GRPO \citep{liu2025flow}& 1.00 &0.99 &0.95 &0.92& 0.99& 0.86 &0.95\\
\midrule
\multicolumn{8}{l}{\textit{\textbf{Unified Understanding \& Generation}}} \\
Show-o~\citep{Xie2024ShowoOS}  & 0.95 & 0.52 & 0.49 & 0.82 & 0.11 & 0.28 & 0.53 \\
\ \ w/ Mask-GRPO \citep{luo2025reinforcement}&  0.99&0.90& 0.69& 0.85& 0.35& 0.59& 0.73\\
Janus-Pro-7B~\citep{Chen2025JanusProUM}  & 0.99 & 0.89 & 0.59 & 0.90 & 0.79 & 0.66 & 0.80 \\
MMaDA~\citep{yang2025mmada}  & 0.96 & 0.60 & 0.45 & 0.81 & 0.14 & 0.25 & 0.56 \\
\ \ w/ UniGRPO\citep{yang2025mmada}  & 0.99 & 0.76 & 0.61 & 0.84 & 0.20 & 0.37 & 0.63 \\
\ \ w/ MaskGRPO \citep{ma2026consolidating}  & 0.99 & 0.85 & 0.66 & 0.89 & 0.73 & 0.69 & 0.80 \\
Lumina-Dimoo \citep{xin2025lumina} &1.00 &0.94& 0.85 &0.89 &0.85& 0.76&0.88\\
\ \ w/ selfGRPO \citep{xin2025lumina} & 1.00& 0.96&0.87& 0.95& 0.85 &0.82& 0.91\\
FUDOKI$^{\dagger}$ \citep{wang2025fudoki} &0.96& 0.85& 0.56& 0.88& 0.68& 0.67& 0.77\\
\rowcolor{morandiblue}\ \ w/ dFlowGRPO  &0.97& 0.93 &0.92& 0.91 &0.88& 0.97& 0.93\\
\bottomrule
\end{tabular}%
}
\end{table}
\subsection{Text-to-Image Generation}\vspace{-3pt}
For text-to-image generation, we evaluate dFlowGRPO's generation capabilities on the GenEval \citep{ghosh2023geneval}, PickScore \citep{kirstain2023pick} and DrawBench \citep{saharia2022photorealistic} benchmarks. We train dFlowGRPO using GenEval and PickScore rewards on their corresponding training splits without KL regularization, respectively. We evaluate dFlowGRPO with $20$ NFEs for each reward. \Cref{tab:geneval} presents the performance of dFlowGRPO on the GenEval benchmark, where the proposed framework improves FUDOKI significantly, achieving the best overall performance (0.93) among the unified multimodal models with fewer denoising steps without CFG. The experimental results of dFlowGRPO trained with PickScore reward are reported in \cref{fig:overview}. We can see that the PickScore and GenEval score increase consistently during the training of dFlowGRPO, improving the PickScore of FUDOKI from 20.87 to 23.00 without evident reward hacking across the reported auxiliary metrics. Additionally, for evaluation results on DrawBench, we observe that dFlowGRPO trained with PickScore reward consistently improves other metrics, including HPSv3 \citep{ma2025hpsv3}, Aesthetic Score \citep{schuhmannlaion}, DeQA \citep{you2025teaching}, ImageReward \citep{xu2023imagereward} and UnifiedReward \citep{wang2025unified}. For dFlowGRPO trained with GenEval reward, the fine-tuned model also improves five rewards evaluated on DrawBench, with only a slight decrease in Aesthetic Score.

\begin{figure}[htbp]
    \centering
    \includegraphics[width=0.5\linewidth]{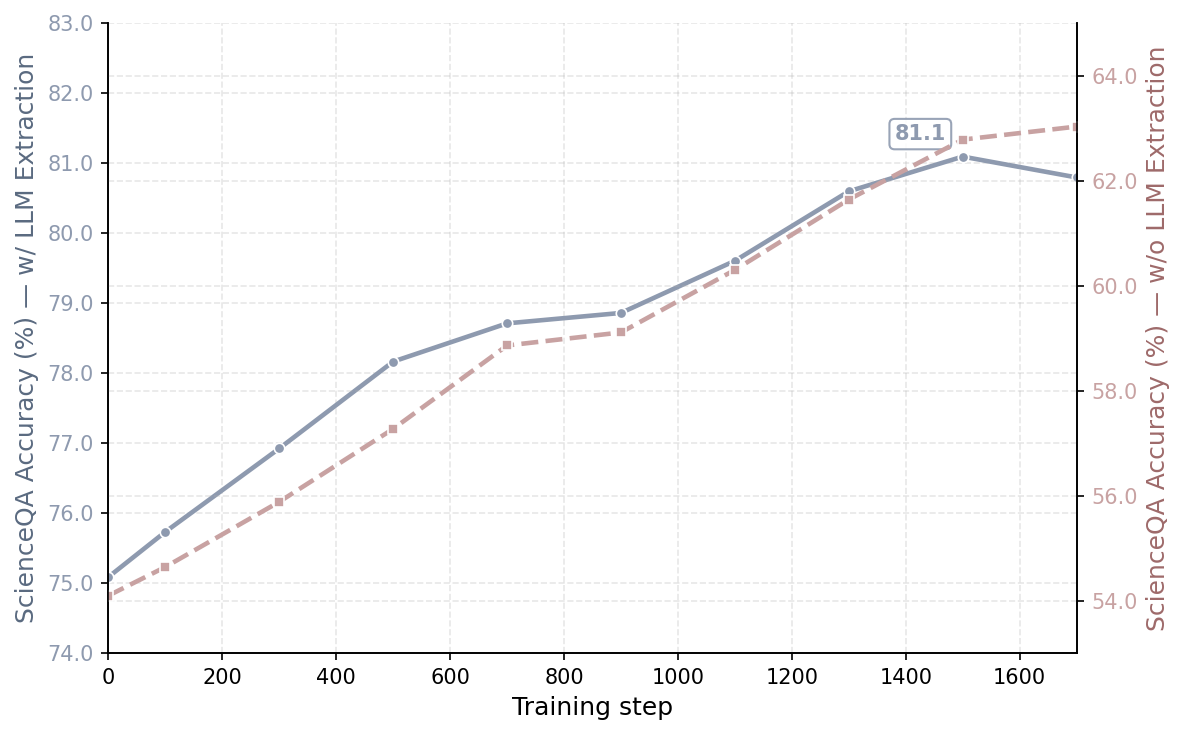}\includegraphics[width=0.35\linewidth]{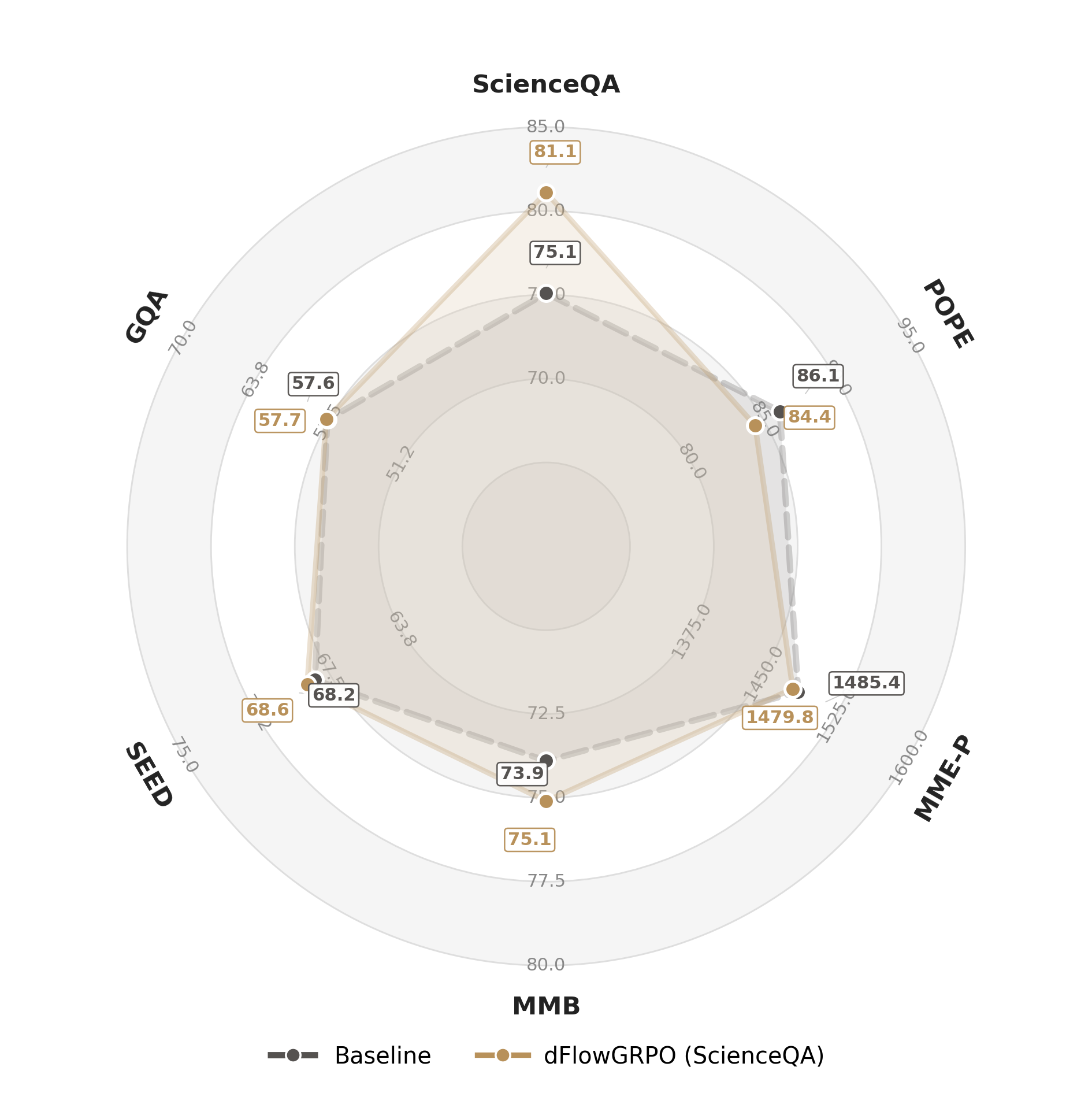}
    
    \caption{\small (Left) Training dynamics of dFlowGRPO trained on the ScienceQA training split; the evaluation results are obtained on the test split. (Right) Evaluation results of dFlowGRPO on several benchmarks.}
    \label{fig:understanding}
\end{figure}
\subsection{Multimodal Understanding}\vspace{-3pt}
We evaluate the understanding capability of dFlowGRPO on six benchmarks, including POPE \citep{li2023evaluating}, MME-P \citep{fumme}, SEED \citep{li2023seed}, MMB \citep{liu2024mmbench}, GQA \citep{hudson2019gqa}, and ScienceQA \citep{lu2022learn}. We fine-tune FUDOKI with dFlowGRPO on the ScienceQA training split with KL regularization $\beta=0.01$, and the 0-1 correctness reward for each answer is obtained by exact matching between the generated answer and the ground truth; otherwise, a judge LLM is used (see \cref{sec:scienceqa_extraction_prompt} for its prompt and error analysis). \Cref{fig:understanding} shows the training dynamics of dFlowGRPO and the evaluation results on multimodal understanding benchmarks. We can observe that the proposed dFlowGRPO improves the accuracy of FUDOKI on ScienceQA from 75.1\% to 81.1\% with LLM extraction and from 54.1\% to 59.1\% without LLM extraction. Additionally, dFlowGRPO achieves better performance than the base model on GQA, SEED and MMB after training, showing limited evidence of reward hacking across the reported auxiliary benchmarks.

\subsection{Ablation Study and Discussion}\vspace{-3pt}
In this subsection, we conduct ablation studies on the effect of KL regularization and the effect of the number of MC samples, and compare dFlowGRPO with the mean-field approximation method for discrete flow models.
Comparison with online DPO, further ablations, evaluation across random seeds, additional benchmarks, unseen-object generalization, the effect of NFEs, and qualitative results are provided in \cref{sec:dflowdpo,sec:additional_ablation_studies,sec:seed_robustness,sec:additional_benchmarks,sec:unseen_object_generalization,sec:effect_of_nfes,sec:qualitative}.
\begin{figure}[htbp]
    \centering
    \includegraphics[width=0.55\linewidth]{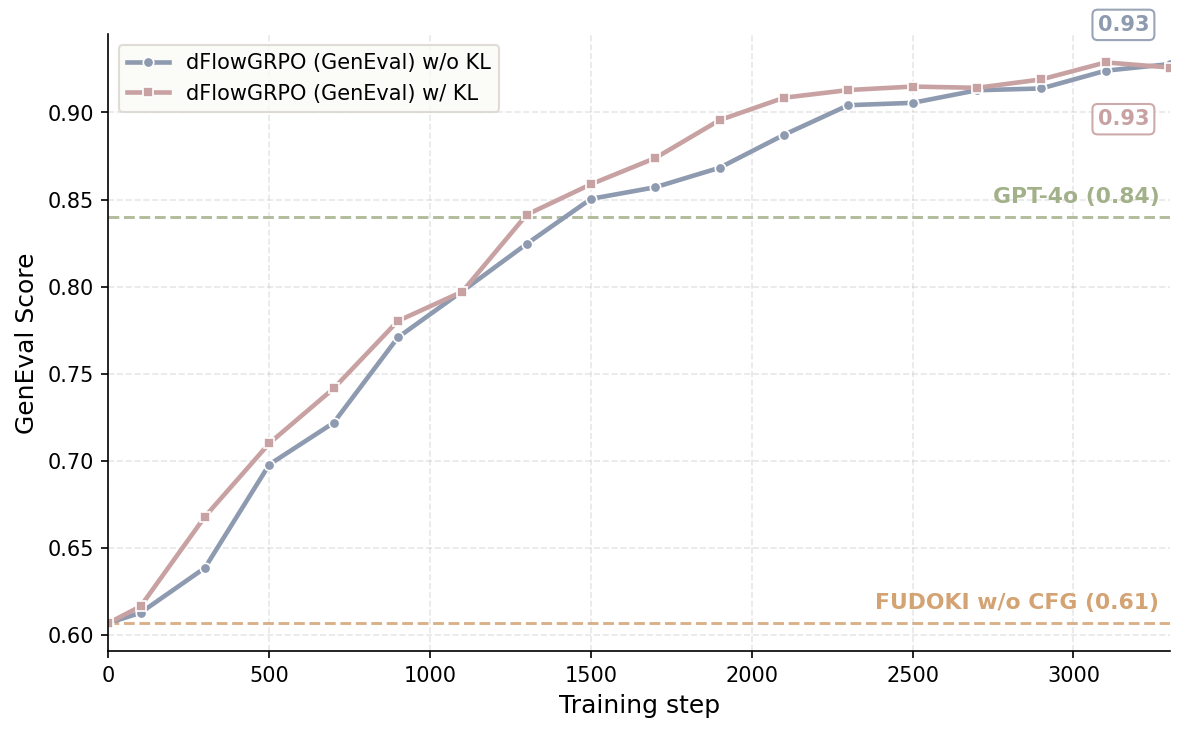}
    \includegraphics[width=0.35\linewidth]{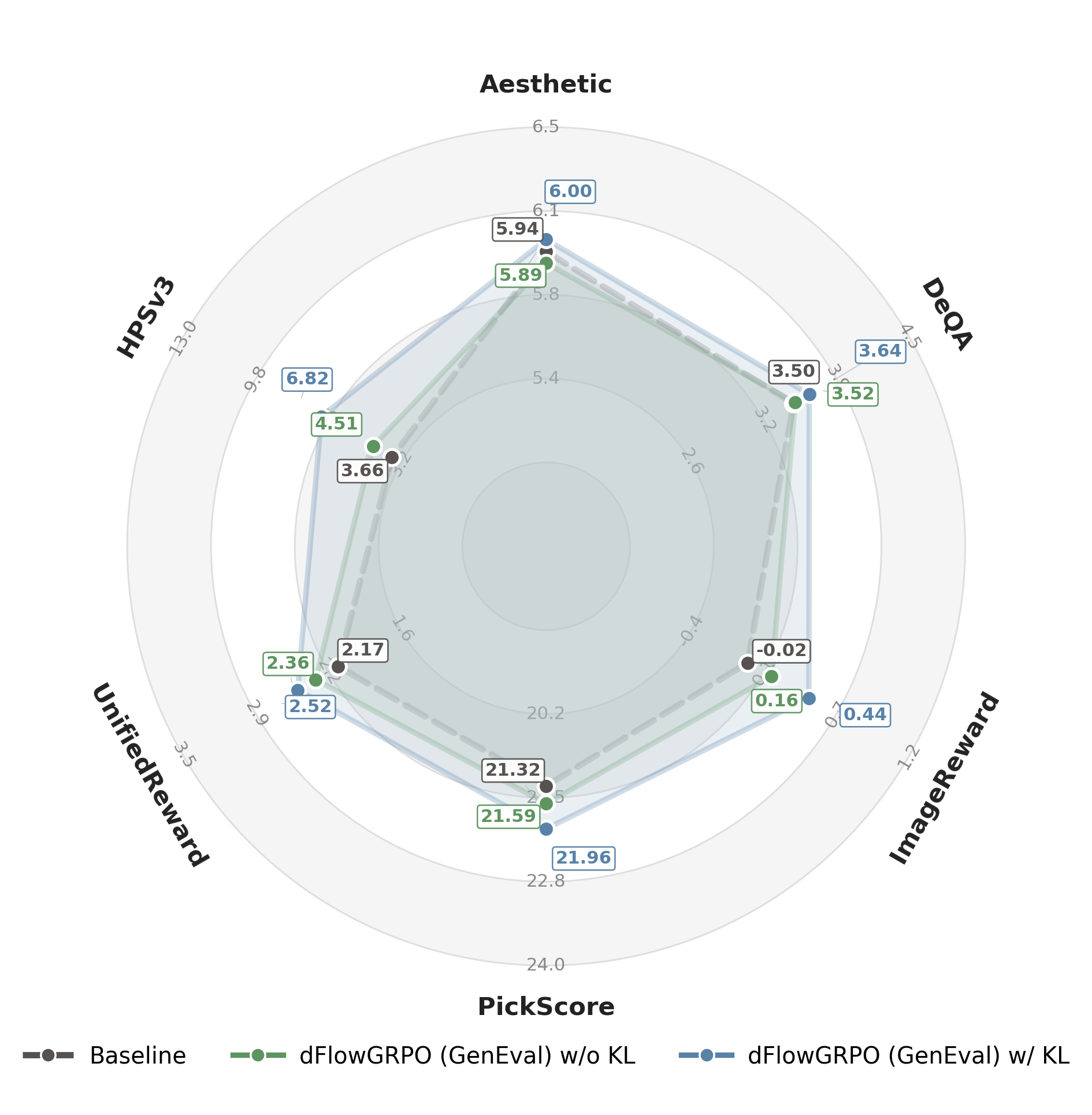}
    \caption{\small (Left) Training dynamics of dFlowGRPO on the GenEval training split (w/ and w/o KL regularization); the evaluation results are obtained on the test split. (Right) Performance of dFlowGRPO evaluated by other metrics on DrawBench.}
    \label{fig:KL regularization}
\end{figure}
\paragraph{Effect of KL regularization.}
We train dFlowGRPO with GenEval reward without and with KL regularization $\beta=0.01$. \Cref{fig:KL regularization} presents the training dynamics of dFlowGRPO in both settings, where the training dynamics are similar and both curves achieve $0.93$ after $3300$ training steps. We also evaluate both models on DrawBench with different reward functions. We can see that KL regularization can mitigate reward hacking and significantly improve most of the other rewards.
\paragraph{Effect of MC samples.}
\begin{figure}[htbp]
    \centering
    \includegraphics[width=0.5\linewidth]{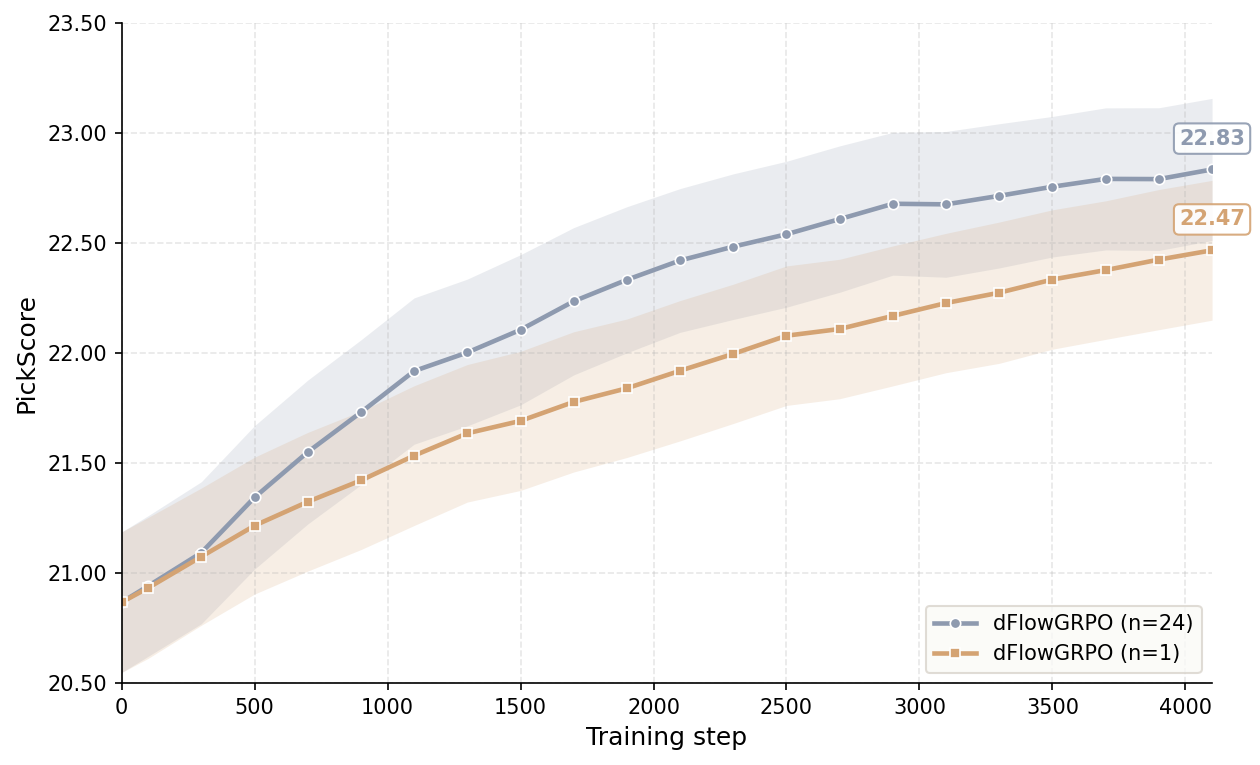}
    \caption{\small Comparison of dFlowGRPO with $n=24$ and $n=1$.}
    \label{fig:mc=1}
\end{figure}
To investigate the effect of the number of MC samples, we train dFlowGRPO with $n=1$ on the PickScore dataset. \Cref{fig:mc=1} shows the training dynamics of dFlowGRPO with $n=1$ and $n=24$ for the first 4100 gradient updates. We can observe that the PickScore reward of dFlowGRPO with $n=24$ achieves 22.50 after 2300 training steps (which requires 29.44 ($\times 8$) H100 GPU hours), but the reward of dFlowGRPO with $n=1$ increases slowly during training and achieves 22.47 after 4100 training steps (which requires 38.13 ($\times 8$) H100 GPU hours), showing that more MC samples can accelerate training and convergence.
\FloatBarrier

\ifdefined\DFlowNexAISLayout
\Needspace{19\baselineskip}
\paragraph{Comparison with mean-field approximation methods.}\mbox{}\par\nobreak
\begingroup
\setlength{\intextsep}{6pt}
\begin{wrapfigure}{r}{0.48\textwidth}
    \centering
    \includegraphics[width=\linewidth]{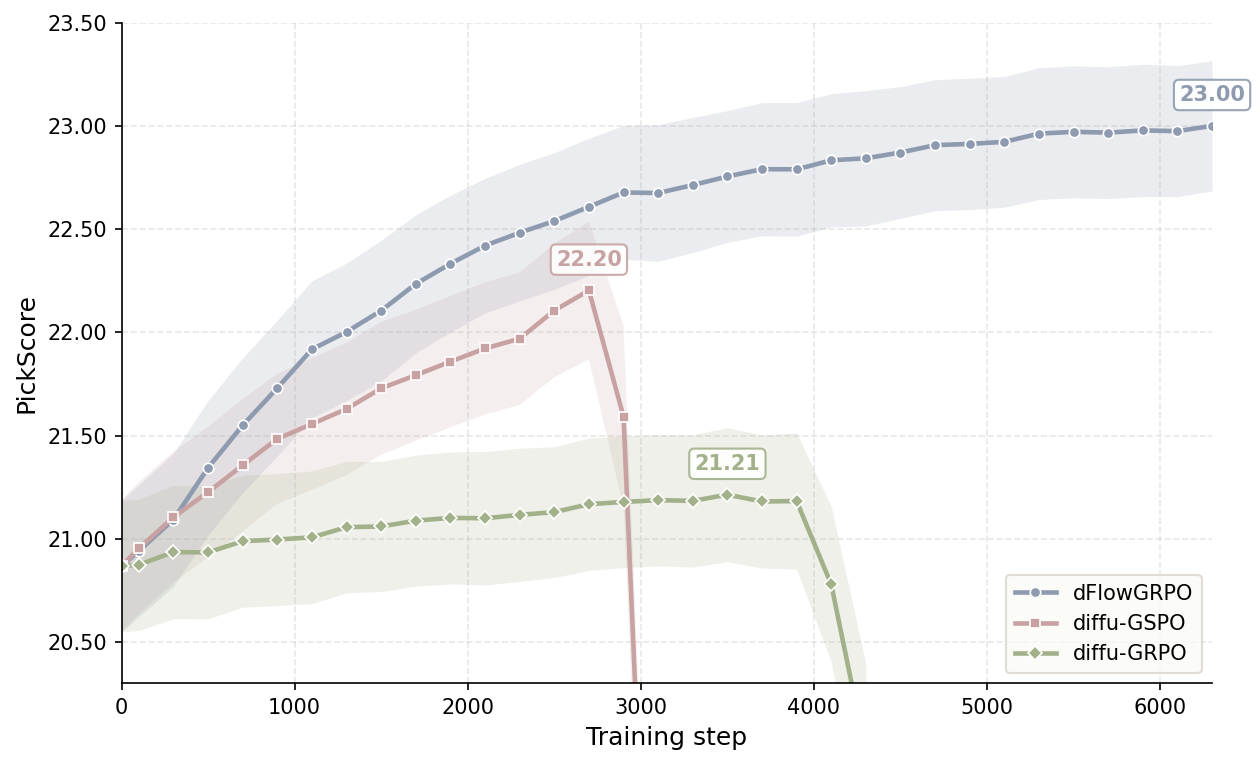}
    \caption{\small Comparison with mean-field approximation methods on PickScore.}
    \label{fig:meanfield approximation}
\end{wrapfigure}
\else
\begin{wrapfigure}[14]{r}{0.5\textwidth}
    \vspace{-18pt}
    \centering
    \includegraphics[width=\linewidth]{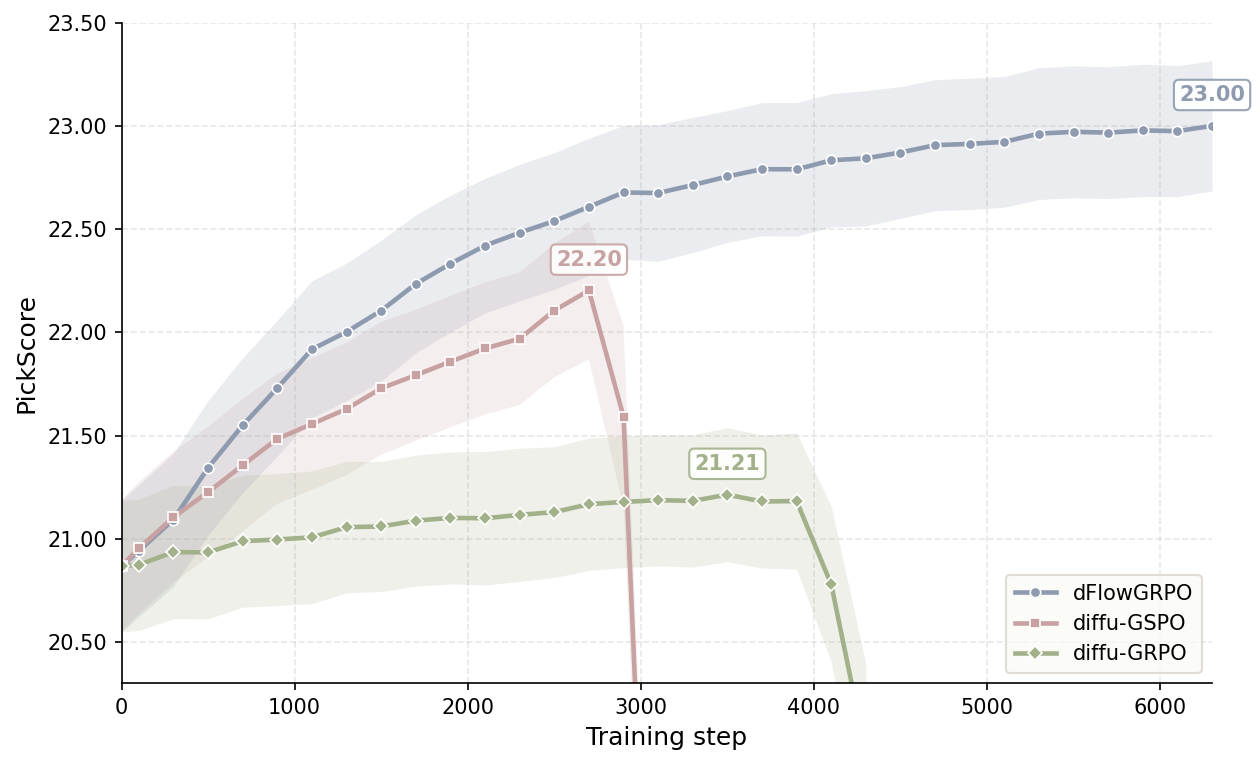}
    \vspace{-18pt}
    \caption{\small Comparison with mean-field approximation methods on PickScore.}
    \label{fig:meanfield approximation}
    \vspace{-24pt}
\end{wrapfigure}
\paragraph{Comparison with mean-field approximation methods.}
\fi
To verify the effectiveness of dFlowGRPO, we compare it with mean-field approximation methods \citep{zhao2025d1} on FUDOKI, a non-masked model. The diffu-GRPO objective and its geometric-mean variant, the diffu-GSPO objective, are defined in \eqref{eq:diffugrpo} and \eqref{eq:diffugspo}, respectively. \Cref{fig:meanfield approximation} presents the training dynamics of dFlowGRPO and two mean-field approximation methods trained with the PickScore reward without KL regularization. For a fair comparison, we use the same number of denoising steps and the same clipping parameters for all three methods. Results show that both mean-field baselines collapse during training, whereas dFlowGRPO improves steadily throughout and achieves a higher reward.
\ifdefined\DFlowNexAISLayout\par\endgroup\fi

\ifdefined\DFlowNexAISLayout\Needspace{12\baselineskip}\fi
\section{Conclusion}\label{sec:conclusion}
In this paper, we introduce dFlowGRPO, a unified RL training framework for DFMs. By deriving the transition probability ratio of the denoising process, dFlowGRPO leverages information from both the conditional rate and the posterior model. We evaluate the proposed method on both text-to-image generation and multimodal understanding tasks, demonstrating the effectiveness and strong performance across various benchmarks.

Building on our framework, one can also apply different rate-aware transition probability ratios to dLLMs by considering alternative conditional rates constructed according to the detailed balance condition \citep[see Appendix F in][]{campbell2024generative}. We leave this interesting direction for future work.

\bibliographystyle{plainnat}
\bibliography{ref}

\beginsupplement
\ifdefined\DFlowNexAISLayout\else\newpage\fi
\section{Algorithms}

In this section, we present the always-valid Euler solver \citep[see \cref{alg:Euler sampler in general} and][]{shaul2024flow,wang2025fudoki} and the proposed dFlowGRPO algorithm (see \cref{alg:dFlowGRPO}).

\begin{algorithm}[htbp]
\caption{Euler solver \citep[Algorithm 1 in][]{shaul2024flow}}
\label{alg:Euler sampler in general}
\begin{algorithmic}[1]
\Require A conditional rate $Q_t^d(x^d,z^d|x_1^d)$, a posterior $p_{1|t}$, time partition $0=t_0<t_1<\cdots<t_K=1$.
\State Draw $\rvx_0\sim p_0$.
\For{$k=0$ to $K-1$}
    \State Set $Z=\rvx_{t_k}$.
    \For{$d=1$ to $\mc{D}$}  
    \State Sample $X_1^d\sim p_{1|t_{k}}^d(\cdot|\rvx_{t_k})$. \textcolor{gray}{(We can sample $n$ i.i.d. copies $\set{X_{j,1}^d}_{j\in\brac{n}}$ of $X_1^d$ and use the empirical mean $\dn\sumn Q^d_{t_{k}}(\rvx_{t_k}^d,z^d|X_{j,1}^d)$ to estimate $Q^d_{t_{k}}(\rvx_{t_k},z^d)$ for variance reduction.)}
    \State Set $\lambda_{k+1}^d=\sum_{z^d\neq \rvx_{t_k}^d}Q^d_{t_{k}}(\rvx_{t_k}^d,z^d|X_1^d)$.
    \State Set $Z^d=\begin{cases}
        z^d,&~\text{with probability } \setBig{1-\exp\parenBig{-(t_{k+1}-t_{k})\lambda_{k+1}^d}}\frac{Q_{t_{k}}^d(\rvx_{t_k}^d,z^d|X_1^d)}{\lambda_{k+1}^d}\\
        Z^d,&~\text{with probability } \exp\parenBig{-(t_{k+1}-t_{k})\lambda_{k+1}^d}
    \end{cases}$, where $z^d\neq Z^d$.
    \EndFor
    \State Set $\rvx_{t_{k+1}}=Z$.
  \EndFor
\State \Return $\rvx_{1}\sim p_{1}$
\end{algorithmic}
\end{algorithm}

\begin{algorithm}[htbp]
\caption{dFlowGRPO}\label{alg:dFlowGRPO}
\begin{algorithmic}[1]
\Require A conditional rate $Q_t^d(x^d,z^d|x_1^d)$, a posterior model $p^\theta_{1|t}$, time partition $0=t_0<t_1<\cdots<t_K=1$, the number of MC samples $n$.
\While{not converged}
\State Sample $\rvc\sim p_c$.
\For{$i=1$ to $G$}
\State Draw $\rvx_{0}^{(i)}\sim p_0$.
\For{$k=0$ to $K-1$}
    \State Set $Z^{(i)}=\rvx_{t_k}^{(i)}$.
    \For{$d=1$ to $\mc{D}$}  
    \State Sample $\set{X_{j,1}^{(i),d}}_{j\in\brac{n}}\overset{i.i.d.}{\sim} p_{1|t_{k}}^{\theta_{old},d}(\cdot|\rvx_{t_k}^{(i)},\rvc)$. 
    \State Set $Q^{(i),d}_{t_{k}}(\rvx_{t_k}^{(i),d},z^d)=\frac{1}{n}\sum_{j\in\brac{n}}Q^{(i),d}_{t_{k}}(\rvx_{t_k}^{(i),d},z^d|X_{j,1}^{(i),d})$
    \State Set $\lambda_{k+1}^{(i),d}=\sum_{z^d\neq \rvx_{t_k}^{(i),d}}Q^{(i),d}_{t_{k}}(\rvx_{t_k}^{(i),d},z^d)$.
    \State Set $Z^{(i),d}=\begin{cases}
        z^d,&~\text{w.p. } \setBig{1-\exp\parenBig{-(t_{k+1}-t_{k})\lambda_{k+1}^{(i),d}}}\frac{Q^{(i),d}_{t_{k}}(\rvx_{t_k}^{(i),d},z^d)}{\lambda_{k+1}^{(i),d}}\\
        Z^{(i),d},&~\text{w.p. } \exp\parenBig{-(t_{k+1}-t_{k})\lambda_{k+1}^{(i),d}}
    \end{cases}$, where $z^d\neq Z^{(i),d}$.
    \EndFor
    \State Set $\rvx_{t_{k+1}}^{(i)}=Z^{(i)}$.
    \State Compute the rate-dependent weight $\hat{w}_{Q,k}^{\theta_{old},d}(\rvx_{t_k}^{(i)},\rvx_{t_{k+1}}^{(i),d},X_{j,1}^{(i),d})$ by \eqref{eq:weight estimator} for each $j\in\brac{n}, d\in\brac{\mc{D}}$.
  \EndFor
\EndFor
\State Set $\hat{A}^{(i)}=\frac{\mc{R}(\rvx_{1}^{(i)},\rvc)-\text{mean}(\set{\mc{R}(\rvx_{1}^{(i)},\rvc)}_{i\in\brac{G}})}{\text{std}(\set{\mc{R}(\rvx_{1}^{(i)},\rvc)}_{i\in\brac{G}})}$.
\State Compute $\hat{r}_k^{(i)}(\theta)=\prod_{d\in\brac{\mc{D}}}\frac{\hat{p}_\theta(\rvx_{t_{k+1}}^{(i),d}|\rvx_{t_k}^{(i)},\rvc)}{\hat{p}_{\theta_{old}}(\rvx_{t_{k+1}}^{(i),d}|\rvx_{t_k}^{(i)},\rvc)}$ for each $k\in\set{0,1,\dots,K-1}$ by \eqref{eq:ratio estimator} and update $\theta$ with dFlowGRPO objective $\mc{J}_{\text{dFlowGRPO}}(\theta)$ \eqref{eq:GRPO_DFM}.
\EndWhile
\end{algorithmic}
\end{algorithm}
\ifdefined\DFlowNexAISLayout\FloatBarrier\fi

\section{Proof of \cref{thm:transition ratio}}\label{sec:proof}
\begin{proof}
Note that for each denoising step, the tokens of $\rvx_{t_{k+1}}$ are independent conditional on the last state $\rvx_{t_k}$; that is, $p_\theta(\rvx_{t_{k+1}}|\rvx_{t_{k}})=\prod_{d\in\brac{\mc{D}}}p_\theta^d(\rvx_{t_{k+1}}^d|\rvx_{t_{k}})$. According to the sampling algorithm (\cref{alg:Euler sampler in general}), the transition probability for the $d$-th token is:
{\small\begin{align*}
        &~p^d_\theta(\rvx_{t_{k+1}}^d|\rvx_{t_k})\\
        =&~\E_{X_1^d\sim p^{\theta,d}_{1|t_{k}}(\cdot|\rvx_{t_{k}})}\Big\{\mathbbm{1}(\rvx_{t_k}^d=\rvx_{t_{k+1}}^d)\parenBig{\exp\parenBig{-(t_{k+1}-t_k)\lambda_{t_k}^d(\rvx_{t_k}^d,X_1^d)}}\\
        &~\quad+\mathbbm{1}(\rvx_{t_k}^d\neq\rvx_{t_{k+1}}^d)\parenBig{\frac{Q^d_{t_k}(\rvx_{t_k}^d,\rvx_{t_{k+1}}^d|X_1^d)}{\lambda_{t_k}^d(\rvx_{t_k}^d,X_1^d)}\setBig{1-\exp\parenBig{-(t_{k+1}-t_k)\lambda_{t_k}^d(\rvx_{t_k}^d,X_1^d)}}}\Big\}\\
        =&~\E_{X_1^d\sim p^{\theta_{old},d}_{1|t_{k}}(\cdot|\rvx_{t_{k}})}\Big\{\mathbbm{1}(\rvx_{t_k}^d=\rvx_{t_{k+1}}^d)\exp\parenBig{-(t_{k+1}-t_k)\lambda_{t_k}^d(\rvx_{t_k}^d,X_1^d)}\frac{p^{\theta,d}_{1|t_{k}}(X_1^d|\rvx_{t_{k}})}{p^{\theta_{old},d}_{1|t_{k}}(X_1^d|\rvx_{t_{k}})}\\
        &~\quad+\mathbbm{1}(\rvx_{t_k}^d\neq\rvx_{t_{k+1}}^d)\frac{Q^d_{t_k}(\rvx_{t_k}^d,\rvx_{t_{k+1}}^d|X_1^d)}{\lambda_{t_k}^d(\rvx_{t_k}^d,X_1^d)}\setBig{1-\exp\parenBig{-(t_{k+1}-t_k)\lambda_{t_k}^d(\rvx_{t_k}^d,X_1^d)}}\frac{p^{\theta,d}_{1|t_{k}}(X_1^d|\rvx_{t_{k}})}{p^{\theta_{old},d}_{1|t_{k}}(X_1^d|\rvx_{t_{k}})}\Big\},
    \end{align*}}
where $\lambda_{t_k}^d(\rvx_{t_k}^d,X_1^d)=\sum_{z^d\neq \rvx_{t_k}^d}Q^d_{t_k}(\rvx_{t_k}^d,z^d|X_1^d)$.

Thus, the transition probability ratio between the current policy and the old policy is
    \begin{align*}
        r_k(\theta)=\prod_{d\in\brac{\mc{D}}}\frac{p^d_\theta(\rvx_{t_{k+1}}^d|\rvx_{t_k})}{p^d_{\theta_{old}}(\rvx_{t_{k+1}}^d|\rvx_{t_k})}=\prod_{d\in\brac{\mc{D}}}\E_{X_1^d\sim p^{\theta_{old}}_{1|t_{k}}(\cdot|\rvx_{t_{k}})} \setBig{w_{Q,k}^{\theta_{old},d}(\rvx_{t_{k}},\rvx_{t_{k+1}}^d,X_1^d)\frac{p^{\theta,d}_{1|t_{k}}(X_1^d|\rvx_{t_{k}})}{p^{\theta_{old},d}_{1|t_{k}}(X_1^d|\rvx_{t_{k}})}},
    \end{align*}
    where the rate-dependent weight $w_{Q,k}^{\theta_{old},d}(\rvx_{t_{k}},\rvx_{t_{k+1}}^d,X_1^d)$ is defined by
{\small    \begin{align*}
       &~w_{Q,k}^{\theta_{old},d}(\rvx_{t_{k}},\rvx_{t_{k+1}}^d,X_1^d)=\frac{1}{p_{\theta_{old}}^d(\rvx_{t_{k+1}}^d|\rvx_{t_k})}\Big\{\mathbbm{1}(\rvx_{t_k}^d=\rvx_{t_{k+1}}^d)\parenBig{\exp\parenBig{-(t_{k+1}-t_k)\lambda_{t_k}^d(\rvx_{t_k}^d,X_1^d)}}\\
        &~\qquad \qquad +\mathbbm{1}(\rvx_{t_k}^d\neq\rvx_{t_{k+1}}^d)\parenBig{\frac{Q^d_{t_k}(\rvx_{t_k}^d,\rvx_{t_{k+1}}^d|X_1^d)}{\lambda_{t_k}^d(\rvx_{t_k}^d,X_1^d)}\setBig{1-\exp\parenBig{-(t_{k+1}-t_k)\lambda_{t_k}^d(\rvx_{t_k}^d,X_1^d)}}}\Big\}. 
    \end{align*}}
The transition probability ratio derived above can be estimated by the empirical mean for a general probability path and the associated conditional rate. In the following, we give an example for the specific choice of the mixture path with conditional rate $Q_t^{d}(x^d,z^d|x_1^d)=\frac{\dot{\kappa}_t}{1-\kappa_t}(\delta_{x_1^d}(z^d)-\delta_{x^d}(z^d))$.

\paragraph{Example (mixture path).}

For the mixture path with conditional rate $Q_t^{d}(x^d,z^d|x_1^d)=\frac{\dot{\kappa}_t}{1-\kappa_t}(\delta_{x_1^d}(z^d)-\delta_{x^d}(z^d))$, we have
\begin{align*}
    \lambda_{t_k}^d(\rvx_{t_k}^d,X_1^d)=\begin{cases}
        0,&~\text{ if }\rvx_{t_k}^d=X_1^d\\
        \frac{\dot{\kappa}_{t_k}}{1-\kappa_{t_k}},&~\text{ if } \rvx_{t_k}^d\neq X_1^d
    \end{cases},
\end{align*}
which implies that
\begin{align*}
    r_k(\theta)=&~\prod_{d\in\brac{\mc{D}}}\frac{p^d_\theta(\rvx_{t_{k+1}}^d|\rvx_{t_k})}{p^d_{\theta_{old}}(\rvx_{t_{k+1}}^d|\rvx_{t_k})}\\
    =&~\prod_{d\in\brac{\mc{D}}}\Big\{\mathbbm{1}(\rvx_{t_k}^d\neq\rvx_{t_{k+1}}^d)\frac{p^{\theta,d}_{1|t_{k}}(\rvx_{t_{k+1}}^d|\rvx_{t_{k}})}{p^{\theta_{old},d}_{1|t_{k}}(\rvx_{t_{k+1}}^d|\rvx_{t_{k}})}\\
    &~\quad+\mathbbm{1}(\rvx_{t_k}^d=\rvx_{t_{k+1}}^d)\frac{p^{\theta,d}_{1|t_{k}}(\rvx_{t_{k+1}}^d|\rvx_{t_{k}})+(1-p^{\theta,d}_{1|t_{k}}(\rvx_{t_{k+1}}^d|\rvx_{t_{k}}))g_Q(k)}{p^{\theta_{old},d}_{1|t_{k}}(\rvx_{t_{k+1}}^d|\rvx_{t_{k}})+(1-p^{\theta_{old},d}_{1|t_{k}}(\rvx_{t_{k+1}}^d|\rvx_{t_{k}}))g_Q(k)}\Big\},
\end{align*}
where $$g_Q(k)=\exp(-(t_{k+1}-t_k)\frac{\dot{\kappa}_{t_k}}{1-\kappa_{t_k}}).$$
In particular, if the source distribution is $\delta_\mask$, then we have
\begin{align*}
    r_k(\theta)=\prod_{d\in\brac{\mc{D}}}\frac{p^d_\theta(\rvx_{t_{k+1}}^d|\rvx_{t_k})}{p^d_{\theta_{old}}(\rvx_{t_{k+1}}^d|\rvx_{t_k})}=\prod_{d: \rvx_{t_k}^d=\mask \text{ and } \rvx_{t_{k+1}}^d\neq\mask}\frac{p^{\theta,d}_{1|t_{k}}(\rvx_{t_{k+1}}^d|\rvx_{t_{k}})}{p^{\theta_{old},d}_{1|t_{k}}(\rvx_{t_{k+1}}^d|\rvx_{t_{k}})},
\end{align*}
which is independent of the time scheduler.

\end{proof}

\section{Derivation of the Gradient}\label{sec:appendix grad}
In this section, we derive the gradient of our proposed dFlowGRPO objective. For simplicity, we omit the clipping operation and take $\beta=0$.

By \cref{thm:transition ratio}, we have
{\small\begin{align*}
    r_k^{(i)}(\theta)=\prod_{d\in\brac{\mc{D}}}\frac{p^d_\theta(\rvx_{t_{k+1}}^{(i),d}|\rvx_{t_k}^{(i)})}{p^d_{\theta_{old}}(\rvx_{t_{k+1}}^{(i),d}|\rvx_{t_k}^{(i)})}=\prod_{d\in\brac{\mc{D}}}\E_{X_1^d\sim p^{\theta_{old}}_{1|t_{k}}(\cdot|\rvx_{t_{k}})} \setBig{w_{Q,k}^{\theta_{old},d}(\rvx_{t_{k}}^{(i)},\rvx_{t_{k+1}}^{(i),d},X_1^d)\frac{p^{\theta,d}_{1|t_{k}}(X_1^d|\rvx_{t_{k}}^{(i)})}{p^{\theta_{old},d}_{1|t_{k}}(X_1^d|\rvx_{t_{k}}^{(i)})}}.
\end{align*}}
Thus, we can obtain that
{\small\begin{align*}
    &~\nabla_\theta\mc{J}_{\text{dFlowGRPO}}(\theta)\\
    =&~ \E\frac{1}{G}\sum_{i=1}^G\frac{1}{K}\sum_{k=0}^{K-1}\hat{A}_{k}^{(i)}\nabla_\theta \brac{r_k^{(i)}(\theta)}^{\frac{1}{\mc{D}}}\\
    =&~\E\frac{1}{G}\sum_{i=1}^G\frac{1}{K}\sum_{k=0}^{K-1}\hat{A}_{k}^{(i)} \brac{r_k^{(i)}(\theta)}^{\frac{1}{\mc{D}}}\parenBig{\frac{1}{\mc{D}}\sum_{d=1}^\mc{D}\nabla_\theta\log p^d_\theta(\rvx_{t_{k+1}}^{(i),d}|\rvx_{t_k}^{(i)})}\\
    =&~
    \E\frac{1}{G}\sum_{i=1}^G\frac{1}{K}\sum_{k=0}^{K-1}\hat{A}_{k}^{(i)} \brac{r_k^{(i)}(\theta)}^{\frac{1}{\mc{D}}}\parenBig{\frac{1}{\mc{D}}\sum_{d=1}^\mc{D}\frac{\E_{X_1^d\sim p^{\theta}_{1|t_{k}}(\cdot|\rvx_{t_{k}})} \setBig{\tilde{w}_{Q,k}^{d}(\rvx_{t_{k}}^{(i)},\rvx_{t_{k+1}}^{(i),d},X_1^d)\nabla_\theta\log p^{\theta,d}_{1|t_{k}}(X_1^d|\rvx_{t_{k}}^{(i)})}}{p^d_\theta(\rvx_{t_{k+1}}^{(i),d}|\rvx_{t_k}^{(i)})}}\\
    =&~\E\frac{1}{G}\sum_{i=1}^G\frac{1}{K}\sum_{k=0}^{K-1}\hat{A}_{k}^{(i)} \brac{r_k^{(i)}(\theta)}^{\frac{1}{\mc{D}}}\parenBig{\frac{1}{\mc{D}}\sum_{d=1}^\mc{D}\E_{X_1^d\sim p^{\theta}_{1|t_{k}}(\cdot|\rvx_{t_{k}})} \setBig{w_{Q,k}^{\theta,d}(\rvx_{t_{k}}^{(i)},\rvx_{t_{k+1}}^{(i),d},X_1^d)\nabla_\theta\log p^{\theta,d}_{1|t_{k}}(X_1^d|\rvx_{t_{k}}^{(i)})}},
\end{align*}}
where we use the definition of the rate-dependent weight (see \eqref{eq:unnormalized weight} and \eqref{eq:normalized weight}) and the expectation is taken over $\rvc\sim p_c, \set{\rvx^{(i)}}_{i=1}^G\sim \pi_{\theta_{old}}(\cdot|\rvc)$. Consequently, the gradient of the dFlowGRPO objective has a form similar to the gradient of standard Flow-GRPO ($\brac{r_k^{(i)}(\theta)}^{\frac{1}{\mc{D}}}\approx 1$), with the policy score replaced by the posterior score weighted by the rate-dependent weight.

\section{Some Discussions and Additional Experiments}
In this section, we discuss other RL methods for DFMs such as mean-field approximation methods \citep{zhao2025d1} and online DPO \citep{guo2024direct}, and conduct additional experiments for these methods. We also provide additional ablation studies and evaluate dFlowGRPO on additional benchmarks, unseen-object generalization, and with different NFEs.

\subsection{diffu-GRPO and diffu-GSPO}
\citet{zhao2025d1} proposed to use mean-field approximation to estimate the log-likelihood for dLLMs in RL training. The corresponding diffu-GRPO is

\begin{equation}\label{eq:diffugrpo}
    \begin{aligned}
    &~\mc{J}_{\text{diffu-GRPO}}(\theta)=\E_{\rvc\sim p_c,\set{\rvx^{(i)}}_{i=1}^G\sim\pi_{\theta_{old}}(\cdot|\rvc)}\\
    &~\qquad\setBig{\frac{1}{G}\sum_{i=1}^G\frac{1}{\mc{D}}\sum_{d=1}^\mc{D}\parenBig{\min(r^{(i)}_d(\theta)\hat{A}^{(i)},\text{clip}(r^{(i)}_d(\theta),1-\eps,1+\eps)\hat{A}^{(i)})-\beta D_{KL}(\pi_\theta\|\pi_{ref})}},
\end{aligned}
\end{equation}
where $r^{(i)}_d(\theta)=\frac{p_\theta(\rvx_{1}^{(i),d}|\rvx_{0}^{(i)},\rvc)}{p_{\theta_{old}}(\rvx_{1}^{(i),d}|\rvx_{0}^{(i)},\rvc)}$ is the token-level posterior ratio at $\rvx_0$. If the source distribution is a masked point mass, then the ratio only depends on $\rvc$ and $\rvx_1^{(i)}$.

Following GSPO \citep{zheng2025group}, we can propose diffu-GSPO to improve the performance of diffu-GRPO by simply replacing the arithmetic mean with the geometric mean of token-level posterior ratios (see \cref{fig:meanfield approximation} for comparison):
{\small\begin{equation}\label{eq:diffugspo}
    \begin{aligned}
    &~\mc{J}_{\text{diffu-GSPO}}(\theta)=\E_{\rvc\sim p_c,\set{\rvx^{(i)}}_{i=1}^G\sim\pi_{\theta_{old}}(\cdot|\rvc)}\\
    &~\quad\setBig{\frac{1}{G}\sum_{i=1}^G\parenBig{\min\parenBig{\bracBig{\prod_{d=1}^\mc{D}r^{(i)}_d(\theta)}^{\frac{1}{\mc{D}}}\hat{A}^{(i)},\text{clip}(\bracBig{\prod_{d=1}^\mc{D}r^{(i)}_d(\theta)}^{\frac{1}{\mc{D}}},1-\eps,1+\eps)\hat{A}^{(i)}}-\beta D_{KL}(\pi_\theta\|\pi_{ref})}}.
\end{aligned}
\end{equation}}
Compared with our proposed dFlowGRPO, the mean-field approximation in diffu-GRPO and diffu-GSPO only uses the information of the samples at the terminal time, leading to inaccurate log-likelihood or trajectory probability estimation. Additionally, it can make RL training unstable or even collapse for a general probability path (see \cref{fig:meanfield approximation}).

\subsection{dFlowDPO: Generalization to Online DPO}
\label{sec:dflowdpo}
In this section, we generalize our approach to online DPO \citep{guo2024direct}.
Consider the following regularized RLHF problem:
\begin{align*}
    \pi^*=\argmax_{\pi_\theta} \E_{\rvc\sim p_c,\rvx\sim\pi_{\theta}(\cdot|\rvc)}\setBig{ \brac{r(\rvx_1,\rvc)}- \beta D_{KL}\parenBig{\pi_\theta(\rvx_{ t_0:t_K}|\rvc)\Big\|\pi_{\text{old}}(\rvx_{t_0:t_K}|\rvc)}}.
\end{align*}
The optimal policy $\pi^*$ has the closed-form $\pi^*(\rvx_{t_0:t_K}|\rvc)\propto \exp(r(\rvx_1,\rvc)/\beta)\pi_{old}(\rvx_{t_0:t_K}|\rvc)$. Plugging $r(\rvx_1,\rvc)=\beta \log \frac{\pi^*(\rvx_{t_0:t_K}|\rvc)}{\pi_{old}(\rvx_{t_0:t_K}|\rvc)}+\log C(\rvc)$ into the log-likelihood of the BT model, we obtain the following online DPO training objective for discrete flow models:
{\small\begin{align*}
    &~\E_{\rvc\sim p_c,\set{\rvx^{(i)}}_{i=1}^G\sim\pi_{\theta_{old}}(\cdot|\rvc)}\setBig{-\log \sigma\parenBig{\beta \log \frac{\pi_\theta(\rvx^+_{t_0:t_K}|\rvc)}{\pi_{old}(\rvx^+_{t_0:t_K}|\rvc)}-\beta \log \frac{\pi_\theta(\rvx^-_{t_0:t_K}|\rvc)}{\pi_{old}(\rvx^-_{t_0:t_K}|\rvc)}}}\\
    \leq&~\E_{\rvc\sim p_c,\set{\rvx^{(i)}}_{i=1}^G\sim\pi_{\theta_{old}}(\cdot|\rvc)}\setBig{-\frac{1}{K}\sum_{k=0}^{K-1}\log \sigma\parenBig{\beta K\setBig{\log \frac{p_\theta(\rvx^+_{t_{k+1}}|\rvx^+_{t_{k}},\rvc)}{p_{\theta_{old}}(\rvx^+_{t_{k+1}}|\rvx^+_{t_{k}},\rvc)}- \log \frac{p_\theta(\rvx^-_{t_{k+1}}|\rvx^-_{t_{k}},\rvc)}{p_{\theta_{old}}(\rvx^-_{t_{k+1}}|\rvx^-_{t_{k}},\rvc)}}}}\\
    \overset{\triangle}{=}&~\mc{J}_{\text{dFlowDPO}}(\theta),
\end{align*}}
where we use Jensen's inequality (for training efficiency, similar to \cite{Wallace_2024_CVPR}), and $\rvx^+_{t_0:t_K}$ and $\rvx^-_{t_0:t_K}$ are the trajectories with the highest and lowest rewards in the group, respectively.

\begin{figure}[htbp]
    \centering
    \includegraphics[width=0.55\linewidth]{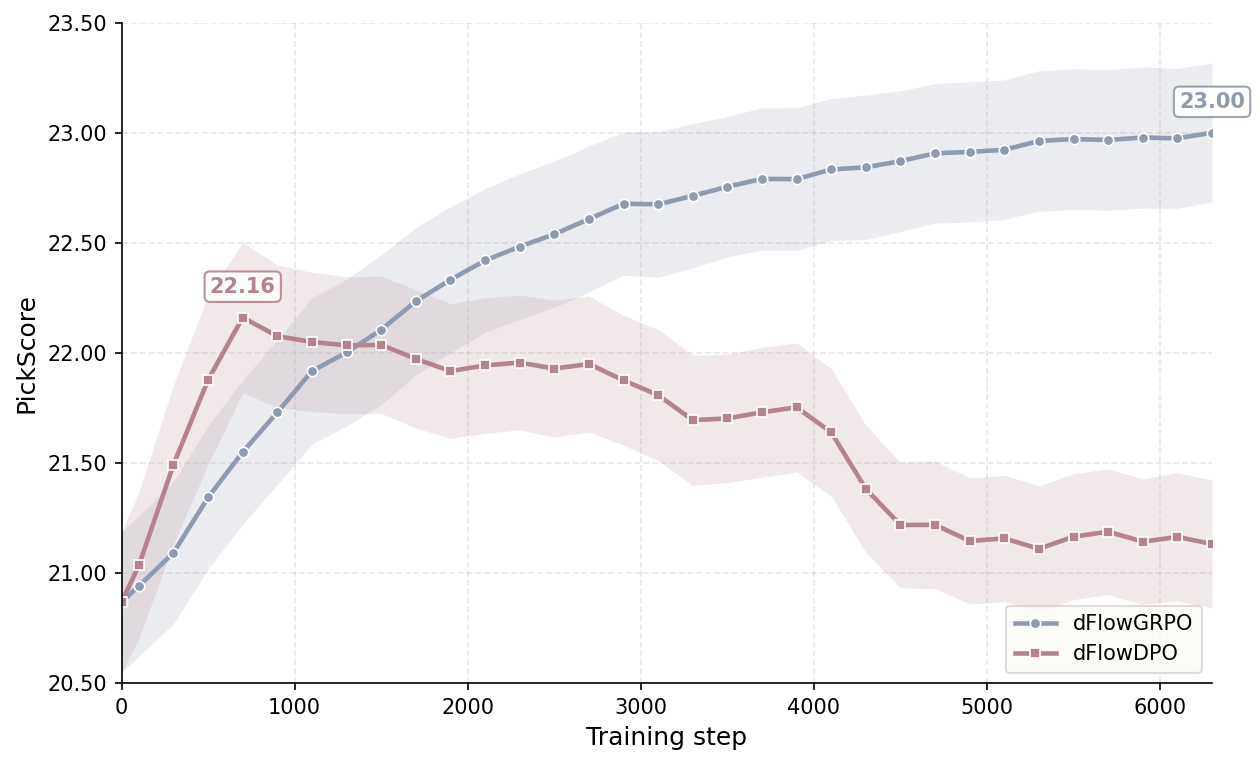}
    \caption{Comparison of dFlowGRPO and dFlowDPO on PickScore.}
    \label{fig:dFlowDPO vs dFlowGRPO}
\end{figure}

We investigate the performance of dFlowDPO on PickScore. For a fair comparison, we set the same number of denoising steps for both dFlowDPO and dFlowGRPO without KL regularization. We choose $\beta=100/K$ for dFlowDPO, as suggested by \citet{liu2025flow}. \Cref{fig:dFlowDPO vs dFlowGRPO} presents the training dynamics of dFlowDPO and dFlowGRPO with PickScore reward, where dFlowDPO increases faster than dFlowGRPO initially and then decreases steadily after 700 training steps.

\FloatBarrier
\subsection{Additional Ablation Studies}
\label{sec:additional_ablation_studies}

We provide the additional experimental results on the number of MC samples and
the rate-dependent weight below. For the ablation results for the
rate-dependent weight, we set the weight in the computation of the transition
probability ratio estimation to 1, and keep the number of MC samples at 24. All
variants below are trained under \textit{an identical compute budget of 13 hours}
on 8 H100 GPUs on PickScore. Due to limited computing resources, this budget is only
intended to verify the effectiveness of the rate-dependent weight rather than to
train the models to convergence.

\begin{table}[htbp]
    \centering
    \caption{\small Additional ablation studies on the number of MC samples and
    the rate-dependent weight.}
    \label{tab:additional_mc_rate_ablation}
    {(a) Number of MC samples}\\[3pt]
    \begin{tabular}{lccccc}
        \toprule
        $n$ & 1 & 4 & 8 & 16 & 24 \\
        \midrule
        PickScore
        & 21.6634
        & 21.8340
        & \textbf{21.8556}
        & 21.7828
        & 21.8243 \\
        \bottomrule
    \end{tabular}

    \vspace{10pt}
    {(b) Rate-dependent weight ($n=24$)}\\[3pt]
    \begin{tabular}{lc}
        \toprule
        Weighting & PickScore \\
        \midrule
        Rate-aware & \textbf{21.8243} \\
        $w=1$ & 21.5541 \\
        \bottomrule
    \end{tabular}
\end{table}
\FloatBarrier

We can see that the performance achieves its best result with
$n=8$, and the rate-dependent weight improves the performance by
0.27 points on PickScore.

\subsection{Results on Additional Benchmarks}
\label{sec:additional_benchmarks}

We have further evaluated the trained models on two additional text-to-image
benchmarks that are not used as training rewards, DPG-Bench (a benchmark used
in the FUDOKI paper) and T2I-CompBench++ (a benchmark used in Flow-GRPO). All
evaluation results below are obtained with 20 NFEs and without CFG, and the
results of the model trained on GenEval with KL regularization are reported.

\begin{table}[htbp]
    \centering
    \caption{\small Evaluation results on additional text-to-image benchmarks.
    The best result in each column is highlighted in bold.}
    \label{tab:additional_t2i_benchmarks}
    \captionsetup[subtable]{font=small,labelfont=normalfont,textfont=normalfont}
    \begin{subtable}{\linewidth}
        \centering
        \caption{DPG-Bench (4 images per prompt).}
        \label{tab:additional_dpg_bench}
        {\normalsize
        \setlength{\tabcolsep}{2pt}
        \begin{tabular}{lcccccc}
            \toprule
            Model & Entity & Attribute & Relation & Other & Global & Overall \\
            \midrule
            FUDOKI & \textbf{85.82} & 84.18 & \textbf{90.95} & \textbf{72.40} & \textbf{82.37} & 77.95 \\
            \ \ w/ dFlowGRPO (PickScore) & 84.59 & 83.74 & 89.44 & 68.40 & 80.24 & \textbf{78.77} \\
            \ \ w/ dFlowGRPO (GenEval) & 85.40 & \textbf{84.28} & 89.48 & 68.00 & 79.64 & 77.86 \\
            \bottomrule
        \end{tabular}}
    \end{subtable}

    \vspace{8pt}
    \begin{subtable}{\linewidth}
        \centering
        \caption{T2I-CompBench++ (10 images per prompt).}
        \label{tab:additional_t2i_compbench}
        {\normalsize
        \setlength{\tabcolsep}{2pt}
        \begin{tabular}{lccccccc}
            \toprule
            Model & Color & Shape & Texture & Spatial & 3D-Spatial & Non-Spatial & Numeracy \\
            \midrule
            FUDOKI & 0.4643 & 0.3488 & 0.4301 & 0.1027 & 0.2164 & 0.2848 & 0.3777 \\
            \ \ w/ dFlowGRPO (GenEval) & \textbf{0.5532} & \textbf{0.3944} & \textbf{0.5065} & \textbf{0.2793} & \textbf{0.3066} & \textbf{0.2924} & \textbf{0.4370} \\
            \bottomrule
        \end{tabular}}
    \end{subtable}
\end{table}

We can see that the trained models perform comparably to the base model on
DPG-Bench, where the model trained with the PickScore reward slightly improves
the overall score and the model trained with the GenEval reward is slightly
worse. One possible main reason is that the
prompts in DPG-Bench are longer and more complex than those in the GenEval and
PickScore training splits. Additionally,
similar to the results in Flow-GRPO, the model trained on GenEval can improve
the performance on T2I-CompBench++ since the prompts in some categories of
T2I-CompBench++, such as color, spatial, and numeracy, are similar to the
GenEval training prompts.

For understanding tasks, we evaluate the model trained on ScienceQA on two
further benchmarks, MMMU and MM-Vet, with 32 NFEs for MMMU and 100 NFEs for
MM-Vet.\footnote{The judge models used by FUDOKI for these two benchmarks have
been retired by OpenAI, so we reevaluate both FUDOKI and our model with the
same judge, \texttt{gpt-4-turbo-2024-04-09}.}

\begin{table}[h!]
    \centering
    \caption{\small Evaluation results on additional multimodal understanding
    benchmarks. The best result in each column is highlighted in bold.}
    \label{tab:additional_understanding_benchmarks}
    {\normalsize
    \setlength{\tabcolsep}{2pt}
    \begin{tabular}{lcc}
        \toprule
        Model & MMMU (validation) & MM-Vet \\
        \midrule
        FUDOKI & 0.3378 & 34.17 \\
        \ \ w/ dFlowGRPO (ScienceQA) & \textbf{0.3533} & \textbf{34.77} \\
        \bottomrule
    \end{tabular}}
\end{table}
\FloatBarrier

These results show no evident reward hacking or catastrophic forgetting on the
two additional understanding benchmarks.

\subsection{Generalization to Unseen Objects}
\label{sec:unseen_object_generalization}
We additionally train a model on a subset of the GenEval object categories and
evaluate it on prompts built from the held-out categories. We partition the 80
GenEval object categories into 60 training categories and 20 test categories
with no overlap. Each split is built only from its own categories, so no object
appearing in the test split is seen in any training prompt. We train dFlowGRPO
for 3000 steps with $\beta_{\mathrm{KL}}=0.01$ and evaluate at 20 NFEs, both
without CFG.

\begin{table}[htbp]
    \centering
    \caption{\small Generalization results on GenEval prompts constructed from
    held-out object categories. The best result in each column is highlighted
    in bold.}
    \label{tab:unseen_object_generalization}
    {\small
    \setlength{\tabcolsep}{2pt}
    \begin{tabular}{lccccccc}
        \toprule
        Model & Single Obj. & Two Obj. & Counting
        & Colors & Position & Attribute & Overall \\
        \midrule
        FUDOKI & 0.7375 & 0.4583 & 0.6627 & 0.7234 & 0.3249 & 0.5585 & 0.5602 \\
        \ \ w/ dFlowGRPO (GenEval) & \textbf{0.7625} & \textbf{0.7760}
        & \textbf{0.8200} & \textbf{0.7926} & \textbf{0.7145}
        & \textbf{0.7606} & \textbf{0.7720} \\
        \bottomrule
    \end{tabular}}
\end{table}
\FloatBarrier

dFlowGRPO improves every category and raises the overall score from 0.5602 to
0.7720. Since no test object appears in any training prompt, these improvements
generalize to objects that are never seen during training.

\subsection{Evaluation Across Random Seeds}
\label{sec:seed_robustness}

The main results are reported at a single evaluation seed. To check that they
are not an artifact of that choice, we reevaluate the base model and our models
at two additional seeds on the full test sets (2048 PickScore prompts, 2212
GenEval prompts and 2015 ScienceQA prompts), at 20 NFEs and without CFG. The
results are reported in \cref{tab:seed_robustness}.

\begin{table}[htbp]
    \centering
    \caption{\small Evaluation results across three random seeds. The first
    seed column is the one reported in the main text (seed 42 for generation,
    seed 99 for understanding).}
    \label{tab:seed_robustness}
    \begin{tabular}{llccc}
        \toprule
        Benchmark & Model & Seed (reported) & Seed 43 & Seed 44 \\
        \midrule
        PickScore & FUDOKI & 20.8672 & 20.8432 & 20.8554 \\
        PickScore & \ \ w/ dFlowGRPO
        & \textbf{23.0016} & \textbf{22.9858} & \textbf{22.9814} \\
        GenEval & FUDOKI & 0.6069 & 0.5959 & 0.6326 \\
        GenEval & \ \ w/ dFlowGRPO
        & \textbf{0.9278} & \textbf{0.9254} & \textbf{0.9294} \\
        ScienceQA & FUDOKI & 0.7509 & 0.7573 & 0.7524 \\
        ScienceQA & \ \ w/ dFlowGRPO
        & \textbf{0.8109} & \textbf{0.8025} & \textbf{0.8089} \\
        \bottomrule
    \end{tabular}
\end{table}
\FloatBarrier

The results are consistent across seeds, and the improvements of dFlowGRPO over
the base model are much larger than the seed-to-seed variation. Multi-seed
training runs are not available due to the computational cost of training large
DFM models.

\subsection{Evaluation with Different NFEs}
\label{sec:effect_of_nfes}
To investigate the effect of NFEs, we report the evaluation results on PickScore and GenEval with different NFEs, which are presented in \cref{tab: effect of nfes}. The results show that dFlowGRPO has a strong performance with few-step generation (e.g., 8 NFEs), achieving the best rewards around 16--32 NFEs.

\begin{table}[htbp]
    \centering
    \caption{\small Evaluation with different NFEs on image generation tasks. The models are trained with corresponding rewards.}
    \vspace{3pt}
    \begin{tabular}{c c c c c c c}
        \toprule
        NFEs & 4 & 8 & 16 & 32 & 48 & 64 \\
        \midrule
        GenEval & 0.46 & 0.91 & 0.93 & 0.93 & 0.93 & 0.92 \\
        PickScore & 21.55 & 22.88 & 22.99 & 22.98 & 22.97 & 22.94 \\
        \bottomrule
    \end{tabular}
    \label{tab: effect of nfes}
\end{table}
\FloatBarrier

\section{Implementation Details}\label{sec:implementation details}
In this section, we provide additional implementation and evaluation details
omitted from the main text. We first describe the training setup and
computational cost, and then present the ScienceQA answer extraction prompt
together with an analysis of extraction errors.

\subsection{Training Setup and Computational Cost}
We use 8 NVIDIA H100 GPUs for all experiments. We fix the group size $G=24$ and the number of MC samples $n=24$ across tasks. We train the model using batches of 16 prompts per gradient update, with gradient clipping at $1.0$ and the AdamW optimizer with $\beta_1=0.9$ and $\beta_2=0.999$. We set the learning rate to $1\times 10^{-5}$ for text-to-image generation and $1\times 10^{-6}$ for multimodal understanding. We use asymmetric clipping parameters $\eps_{\text{high}}=1.5\cdot\eps_{\text{low}}=1.5\times 10^{-3}$ for the GRPO objective by searching from $\set{1.5\times 10^{-2},1.5\times 10^{-3},1.5\times 10^{-4}}$, following the clip-higher strategy proposed in DAPO \citep{yu2026dapo}. For training stability, we update the old policy once per 48 gradient updates. We train dFlowGRPO for 3300 steps on GenEval, 6300 steps on PickScore, and 1500 steps on ScienceQA. We report the computational cost of our experiments in \cref{tab:computational_cost}.

\begin{table}[htbp]
    \centering
    \caption{\small The wall-clock time of different settings for the first 100 gradient updates on 8 H100 GPUs.}
    \label{tab:computational_cost}
    \vspace{3pt}
    \begin{tabular}{lc}
        \toprule
        Method & Time (hours) \\
        \midrule
        dFlowGRPO (ScienceQA) & $2.56$ \\
        dFlowGRPO (PickScore) & $1.28$ \\
        dFlowGRPO (GenEval) & $1.91$ \\
        dFlowGRPO (GenEval, w/ KL) & $2.02$ \\
        dFlowGRPO (PickScore, $n=1$) & $0.93$ \\
        diffu-GRPO & $0.87$ \\
        diffu-GSPO & $0.83$ \\
        dFlowDPO & $0.84$ \\
        \bottomrule
    \end{tabular}
\end{table}
\FloatBarrier

\subsection{ScienceQA Answer Extraction Prompt}
\label{sec:scienceqa_extraction_prompt}

We follow the VLMEvalKit protocol. The answer letter is first extracted by
exact rule-based matching; the LLM judge is called only when that fails. The
prompt is copied from VLMEvalKit:

\begin{center}
\begingroup
\setlength{\fboxsep}{8pt}
\fbox{%
\begin{minipage}{0.92\linewidth}
\small\ttfamily\raggedright\sloppy
\setlength{\parindent}{0pt}
You are an AI assistant who will help me to match an answer with several
options of a single-choice question. You are provided with a question, several
options, and an answer, and you need to find which option is most similar to
the answer. If the meaning of all options are significantly different from the
answer, output Z. Your should output a single uppercase character in A, B, C,
D (if they are valid options), and Z.

Example 1:\\
Question: What is the main object in image?\\
Options: A. teddy bear B. rabbit C. cat D. dog\\
Answer: a cute teddy bear\\
Your output: A

Example 2:\\
Question: What is the main object in image?\\
Options: A. teddy bear B. rabbit C. cat D. dog\\
Answer: Spider\\
Your output: Z

Example 3:\\
Question: \{question\}?\\
Options: \{options\}\\
Answer: \{model output\}\\
Your output:
\end{minipage}}
\endgroup
\end{center}

On the ScienceQA test split (2015 prompts), we checked every question on which
the judge model was called and counted how often it read the wrong letter. The
results are reported in \cref{tab:scienceqa_extraction_errors}. We can also see
that our model can generate more answers that can be extracted by rule-based
matching, which is consistent with the higher accuracy of our model.

\begin{table}[htbp]
    \centering
    \caption{\small Answer-extraction statistics on the ScienceQA test split.}
    \label{tab:scienceqa_extraction_errors}
    \begin{tabular}{lcc}
        \toprule
        Model & LLM judge calls & Extraction errors \\
        \midrule
        FUDOKI & 52 (2.58\%) & 9 \\
        \ \ w/ dFlowGRPO & 25 (1.24\%) & 2 \\
        \bottomrule
    \end{tabular}
\end{table}
\FloatBarrier

\section{Additional Experiments on a Second DFM Backbone}
\label{sec:appendix_second_dfm}

To further demonstrate the effectiveness and robustness of our approach, we
apply dFlowGRPO to a second base model, FS-DFM 1.3B\footnote{Hugging Face: \url{https://huggingface.co/aminr8/FS-DFM}.}
\citep{monsefi2025fs}, a text-only DFM built on a mixture path with a uniform
source and the conditional rate
\begin{align*}
    q^d_{t\mid 1}(x^d\mid x_1^d)
    &=
    \kappa_t\delta_{x_1^d}(x^d)
    +\frac{1-\kappa_t}{\lvert\mathcal{S}\rvert},
    \\
    Q_t^d(x^d,z^d\mid x_1^d)
    &=
    \frac{\dot{\kappa}_t}{1-\kappa_t}
    \left[
        \delta_{x_1^d}(z^d)-\delta_{x^d}(z^d)
    \right],
\end{align*}
with $\kappa_t=t$. 

The task is review completion on the IMDB dataset
\citep{maas2011learning}, a standard corpus of 50k movie reviews annotated
with binary (positive/negative) sentiment; we draw prompts from its labelled
train/test splits (10k / 2k). The prompt is a three-word phrase
(``Positive movie review:'' or ``Negative movie review:'') followed by the
first 32 tokens of a real review, and the model generates the next 32 tokens.
Note that only the generated 32 tokens are scored, and the goal is for their
sentiment to match the review's label. The reward is
\begin{align*}
    \mathcal{R}
    =
    P(\text{target sentiment}\mid\text{generated text})
    +\operatorname{distinct}_1(\text{generated text}),
\end{align*}
where $P(\cdot)$ is given by a sentiment classifier\footnote{Hugging Face: \url{https://huggingface.co/distilbert/distilbert-base-uncased-finetuned-sst-2-english}.} and
$\operatorname{distinct}_1$ is the ratio of unique unigrams to total unigrams
in the generated text, which penalizes degenerate repetition.

In our implementation, we set the number of MC samples to 24 and the learning
rate to $5\mathrm{e}{-6}$ with $\beta_{\mathrm{KL}}=0.01$, and keep the
optimizer settings the same as in the paper's experiments. We train the model
on the training set for 1000 training steps and report the evaluation results
on the test set below. The number of sampling steps is set to 20. dFlowGRPO
significantly improves the performance, further demonstrating the effectiveness
and superiority of the proposed method.

\ifdefined\DFlowNexAISLayout
\begin{table}[!htbp]
\else
\begin{table}[t]
\fi
    \centering
    \caption{Evaluation results of dFlowGRPO on FS-DFM 1.3B for
    sentiment-controlled review completion on the IMDB test set.}
    \label{tab:fs_dfm_imdb}
    \resizebox{\linewidth}{!}{%
    \begin{tabular}{lccccc}
        \toprule
        Model
        & Success rate
        & Success (pos.\ label)
        & Success (neg.\ label)
        & $\operatorname{distinct}_1$
        & Reward \\
        \midrule
        FS-DFM 1.3B 
        & 0.6007 & 0.4479 & 0.7535 & 0.9223 & 1.5214 \\
        \rowcolor{morandiblue}\ \ w/ dFlowGRPO
        & 0.9915
        & 0.9960
        & 0.9870
        & 0.9961
        & 1.9870 \\
        \bottomrule
    \end{tabular}}
\end{table}

We can see that dFlowGRPO raises the reward from 1.5214 to 1.9870 and the
success rate from 0.6007 to 0.9915.

\FloatBarrier

\section{Qualitative Results}\label{sec:qualitative}
In this section, we present qualitative results, including visualizations of the sampling trajectories during training and comparisons between the base model and the model after dFlowGRPO training. 
\subsection{Trajectory Results}
\Cref{fig:trajectory1,fig:trajectory2} show how the denoising trajectories change over the course of training. We can see that dFlowGRPO training enables the model to form more details in the early denoising steps in the course of training.
\begin{figure}[htbp]
    \centering
    \includegraphics[width=0.7\linewidth]{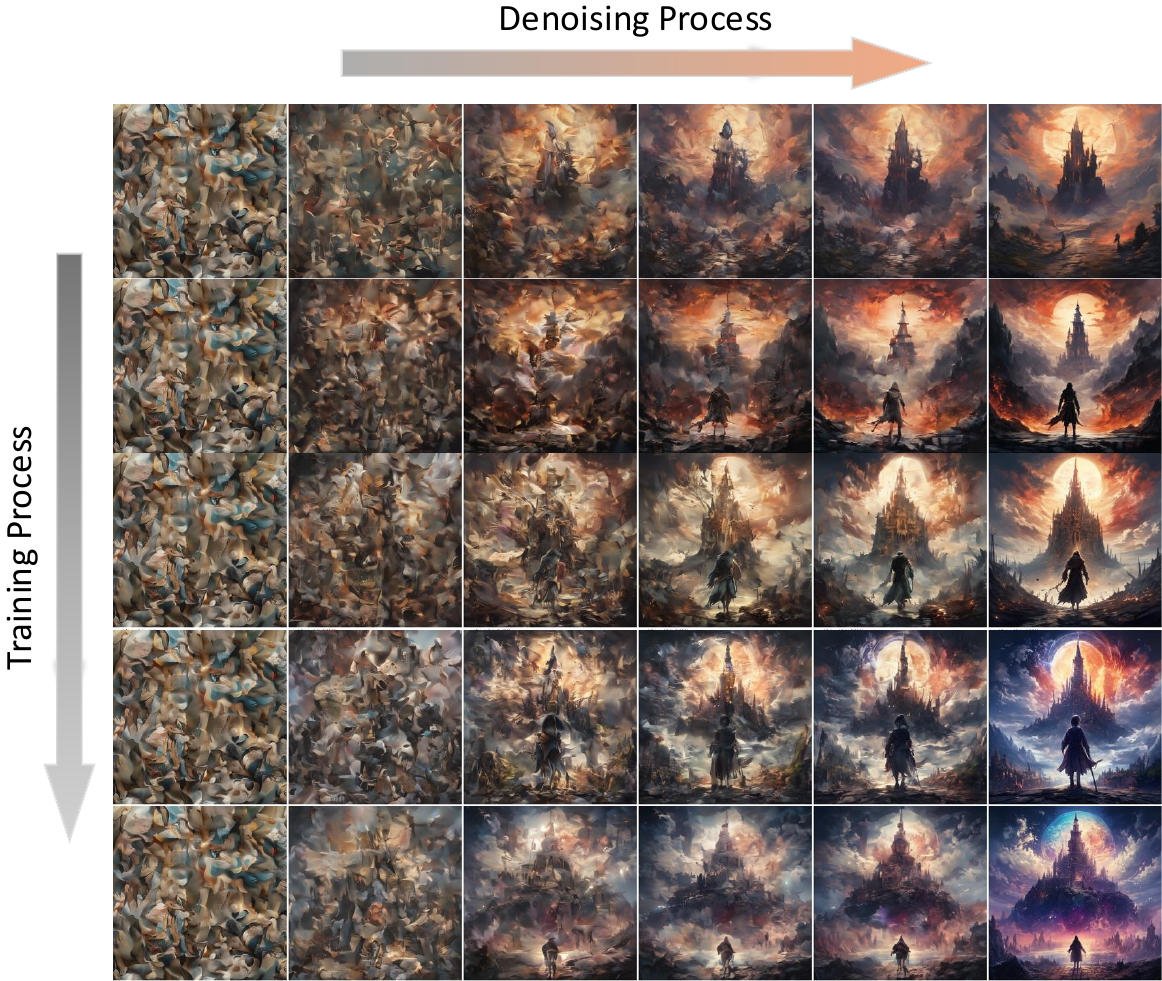}
    
    \vspace{8pt}
    \includegraphics[width=0.7\linewidth]{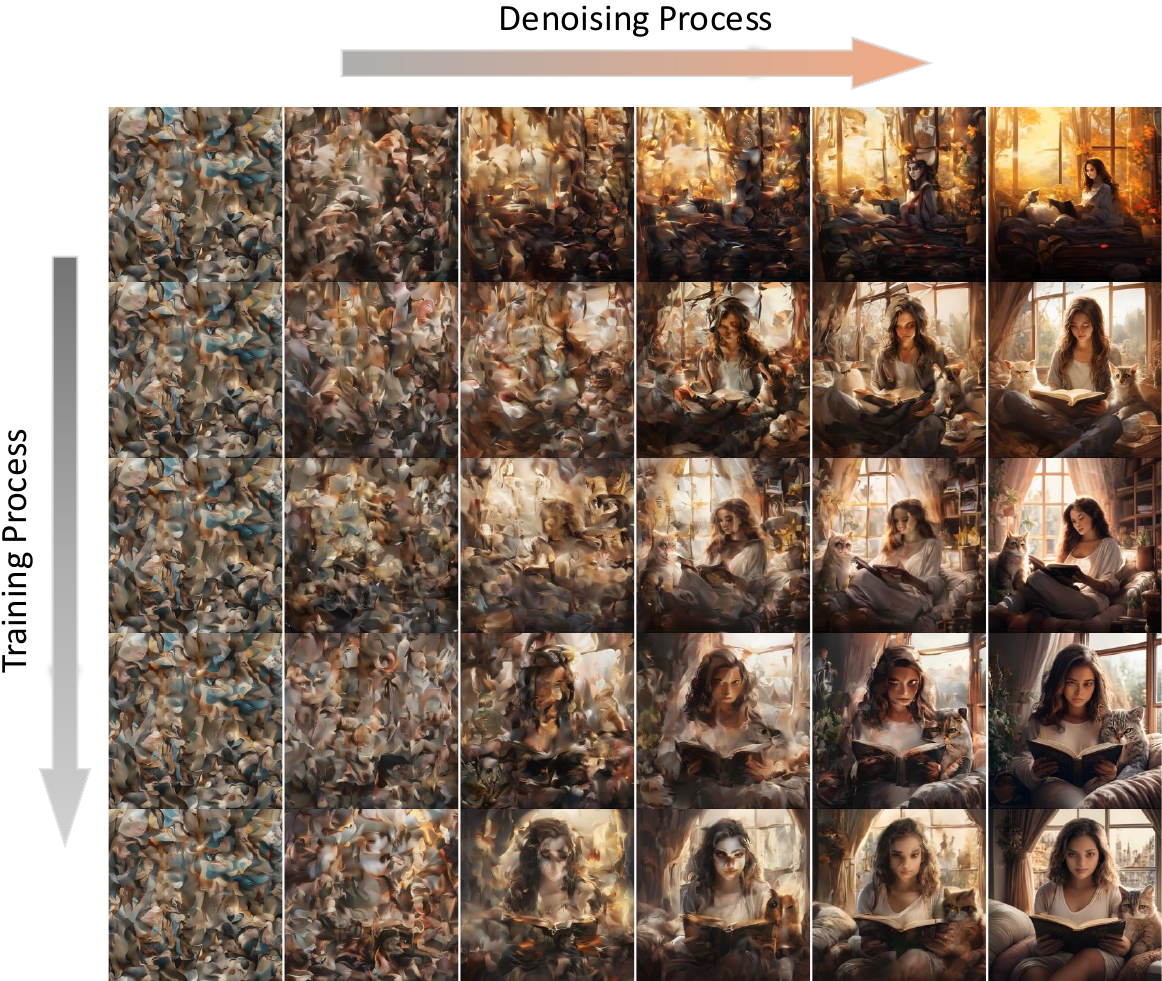}
    \caption{\small Visualization of the trajectory of dFlowGRPO during training with PickScore reward.}
    \label{fig:trajectory1}
\end{figure}

\begin{figure}[htbp]
    \centering
    \includegraphics[width=0.7\linewidth]{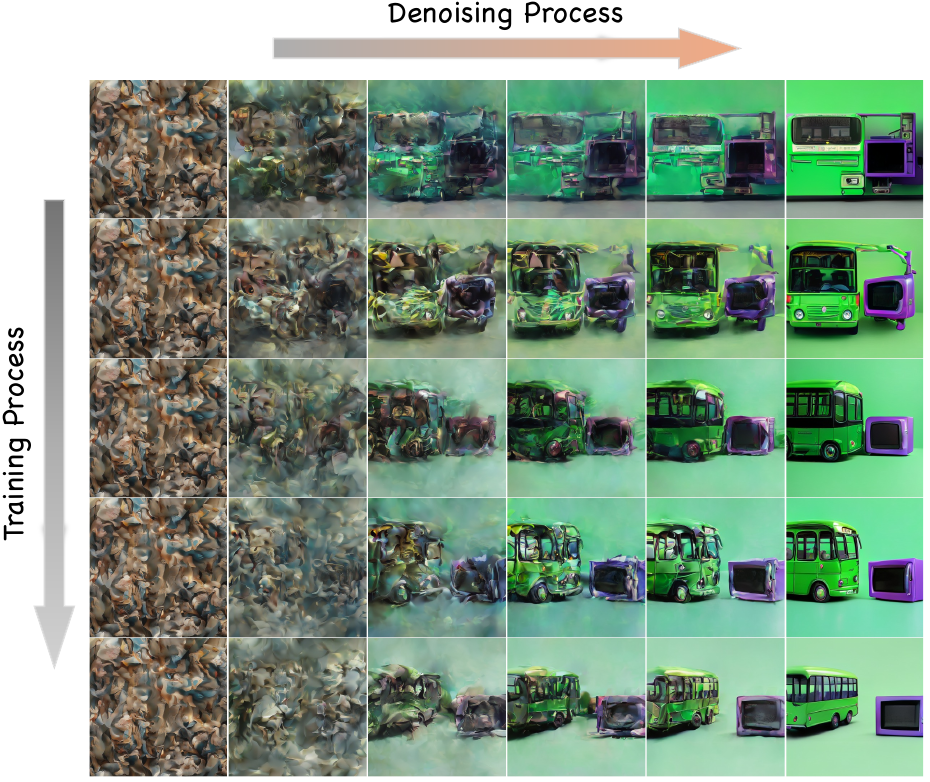}
    
    \vspace{8pt}
    \includegraphics[width=0.7\linewidth]{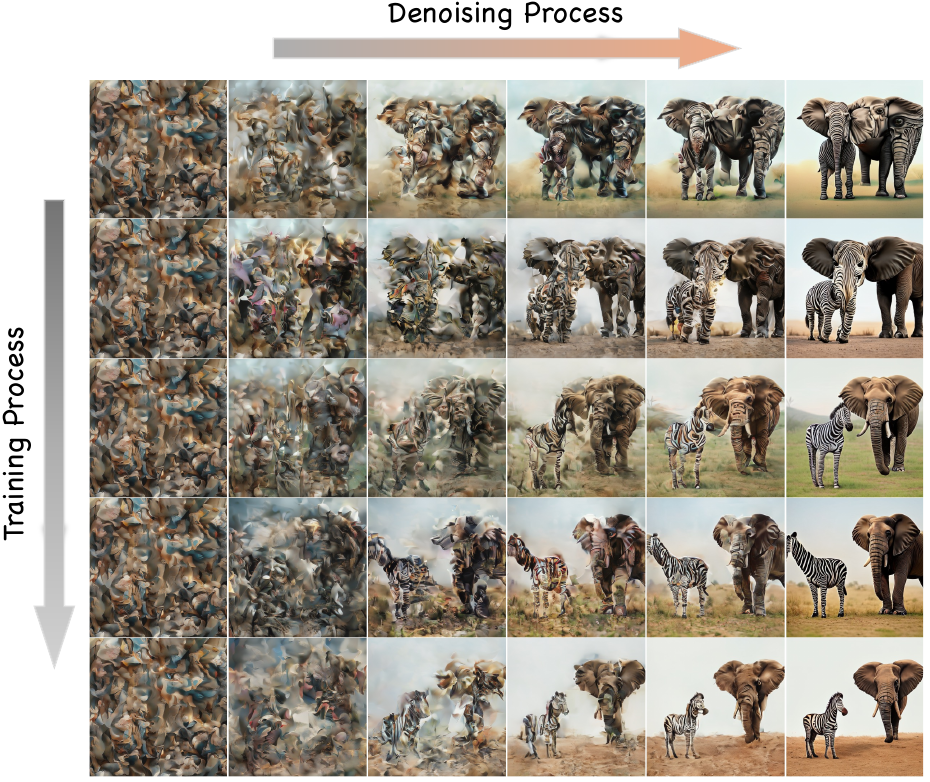}
    \caption{\small Visualization of the trajectory of dFlowGRPO during training with GenEval reward.}
    \label{fig:trajectory2}
\end{figure}

\subsection{Results at Terminal Time}
We present the qualitative results for dFlowGRPO and the base model, FUDOKI, on both image generation and multimodal understanding tasks in \cref{fig:generation quality,fig:understanding quality1,fig:understanding quality2,fig:understanding quality3,fig:understanding quality4}. We can see that the generation quality and multimodal understanding capability of dFlowGRPO significantly outperform those of the base model.
\begin{figure}[htbp]
    \centering

    \begin{subfigure}{0.16\textwidth}
        \centering
        \includegraphics[width=\linewidth]{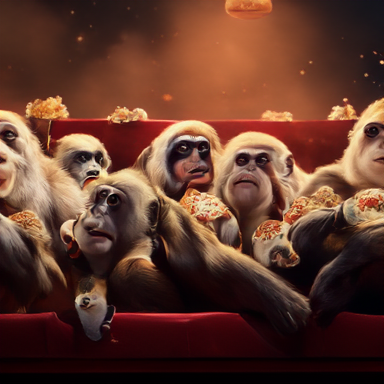}
    \end{subfigure}
    \hfill
    \begin{subfigure}{0.16\textwidth}
        \centering
        \includegraphics[width=\linewidth]{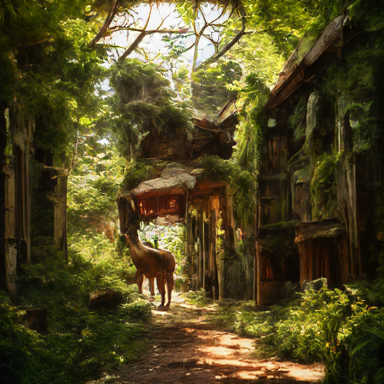}
    \end{subfigure}
    \hfill
    \begin{subfigure}{0.16\textwidth}
        \centering
        \includegraphics[width=\linewidth]{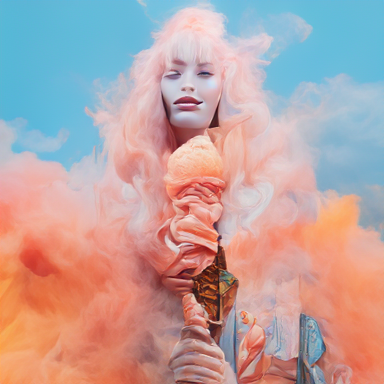}
    \end{subfigure}
    \hfill
    \begin{subfigure}{0.16\textwidth}
        \centering
        \includegraphics[width=\linewidth]{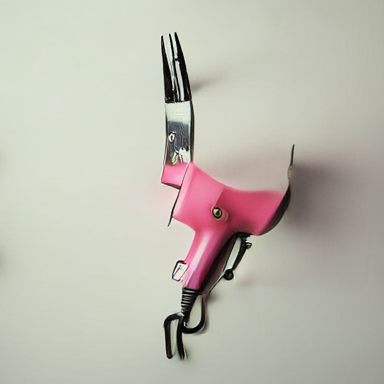}
    \end{subfigure}
    \hfill
    \begin{subfigure}{0.16\textwidth}
        \centering
        \includegraphics[width=\linewidth]{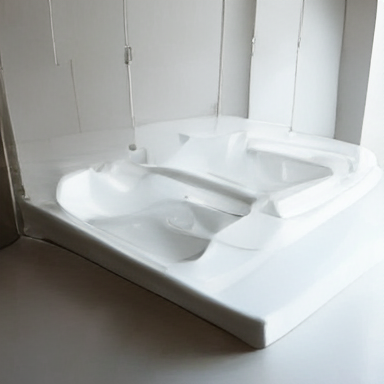}
    \end{subfigure}
    \hfill
    \begin{subfigure}{0.16\textwidth}
        \centering
        \includegraphics[width=\linewidth]{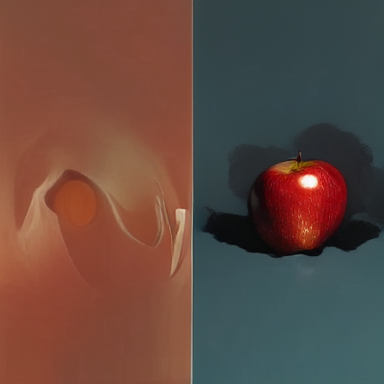}
    \end{subfigure}

    \vspace{0.5em}

    \begin{subfigure}{0.16\textwidth}
        \centering
        \includegraphics[width=\linewidth]{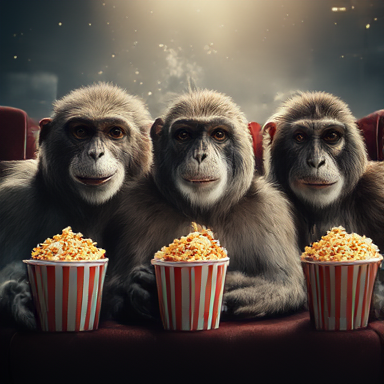}
    \end{subfigure}
    \hfill
    \begin{subfigure}{0.16\textwidth}
        \centering
        \includegraphics[width=\linewidth]{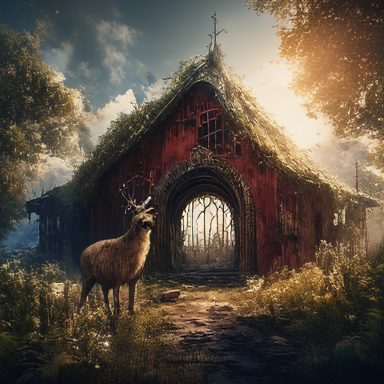}
    \end{subfigure}
    \hfill
    \begin{subfigure}{0.16\textwidth}
        \centering
        \includegraphics[width=\linewidth]{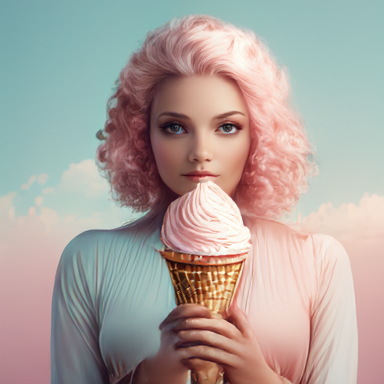}
    \end{subfigure}
    \hfill
    \begin{subfigure}{0.16\textwidth}
        \centering
        \includegraphics[width=\linewidth]{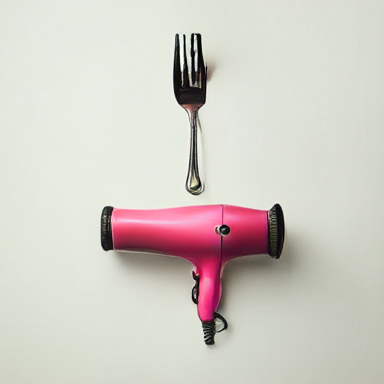}
    \end{subfigure}
    \hfill
    \begin{subfigure}{0.16\textwidth}
        \centering
        \includegraphics[width=\linewidth]{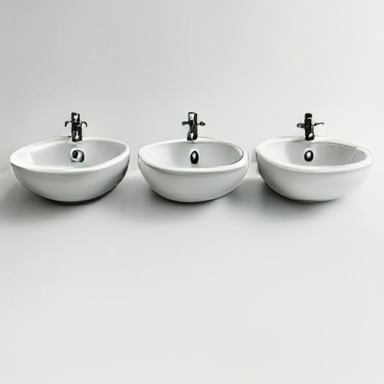}
    \end{subfigure}
    \hfill
    \begin{subfigure}{0.16\textwidth}
        \centering
        \includegraphics[width=\linewidth]{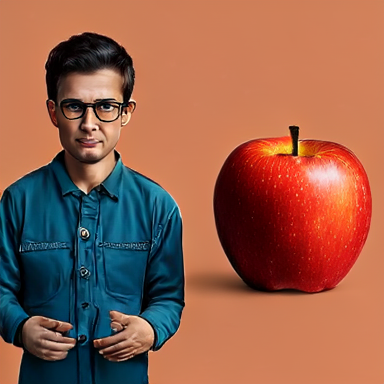}
    \end{subfigure}

    \caption{\small Qualitative comparison of generation quality between the base model FUDOKI (top) and dFlowGRPO (bottom).}
    \label{fig:generation quality}
\end{figure}

\begin{figure}
    \centering
    \includegraphics[width=1\linewidth]{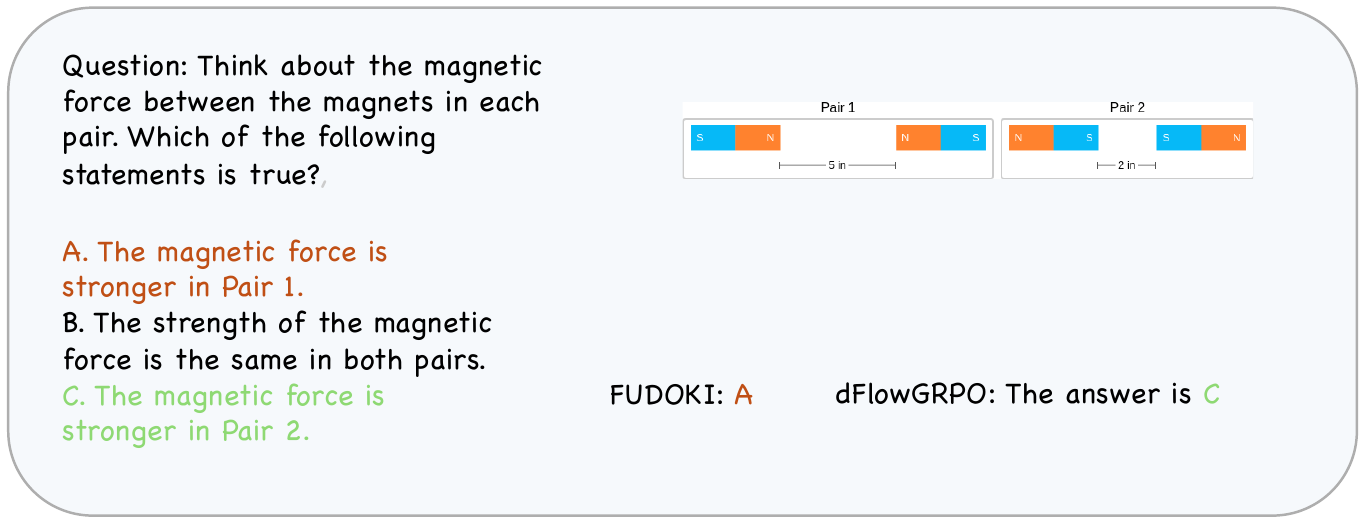}
    \caption{\small Qualitative comparison of understanding capability between dFlowGRPO and the base model FUDOKI. The correct answer is highlighted in green and the wrong answer is highlighted in red.}
    \label{fig:understanding quality1}
\end{figure}

\begin{figure}
    \centering
    \includegraphics[width=1\linewidth]{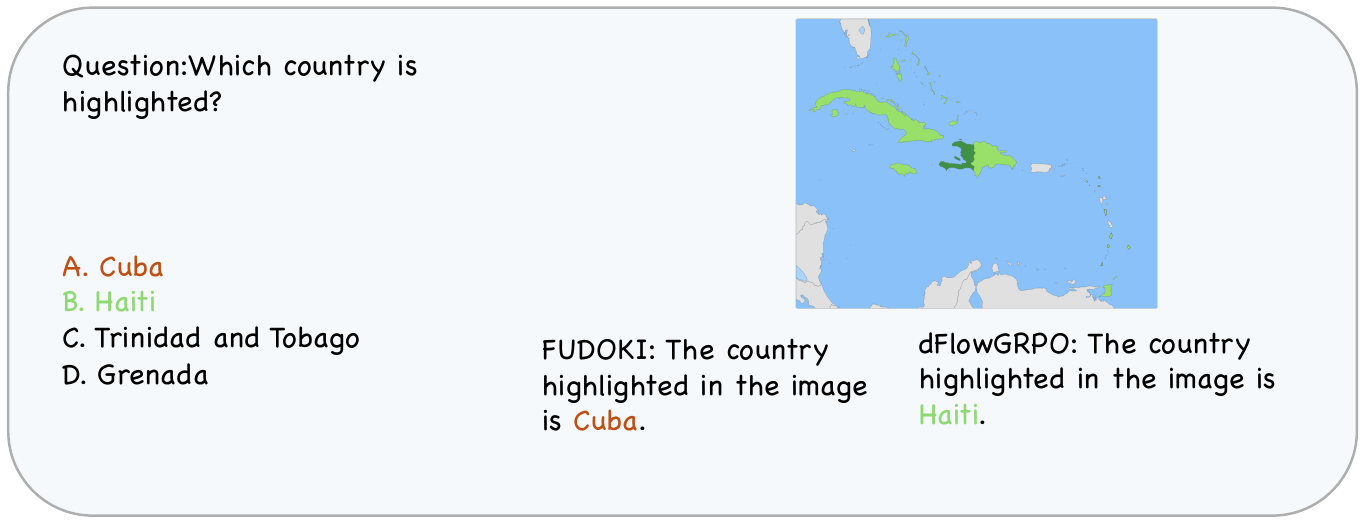}
    \caption{\small Qualitative comparison of understanding capability between dFlowGRPO and the base model FUDOKI. The correct answer is highlighted in green and the wrong answer is highlighted in red.}
    \label{fig:understanding quality2}
\end{figure}

\begin{figure}
    \centering
    \includegraphics[width=1\linewidth]{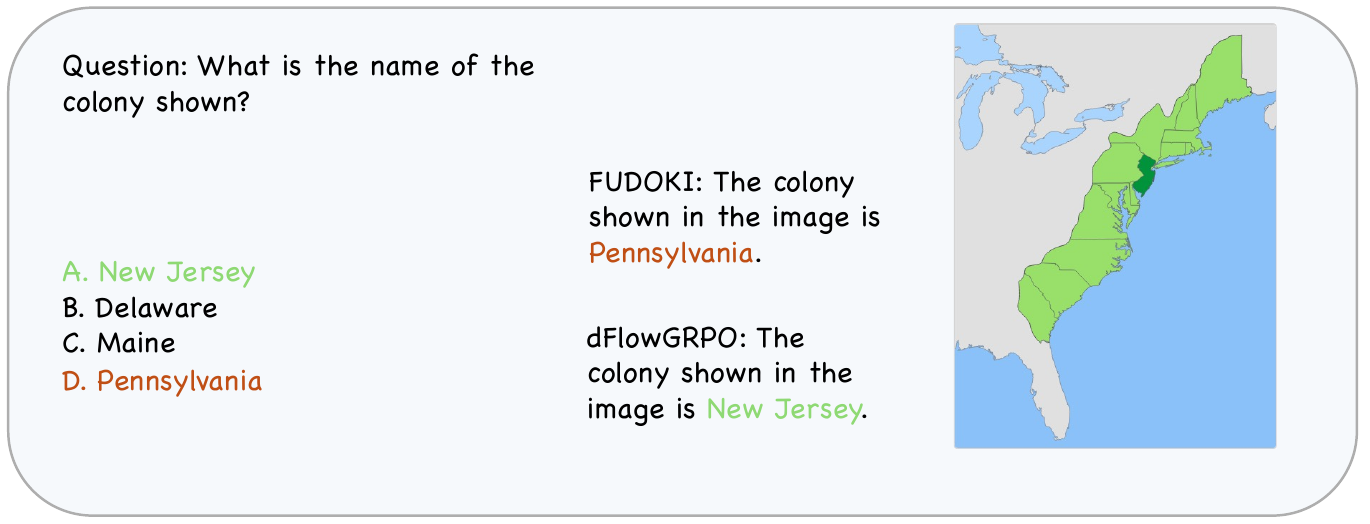}
    \caption{\small Qualitative comparison of understanding capability between dFlowGRPO and the base model FUDOKI. The correct answer is highlighted in green and the wrong answer is highlighted in red.}
    \label{fig:understanding quality3}
\end{figure}

\begin{figure}
    \centering
    \includegraphics[width=1\linewidth]{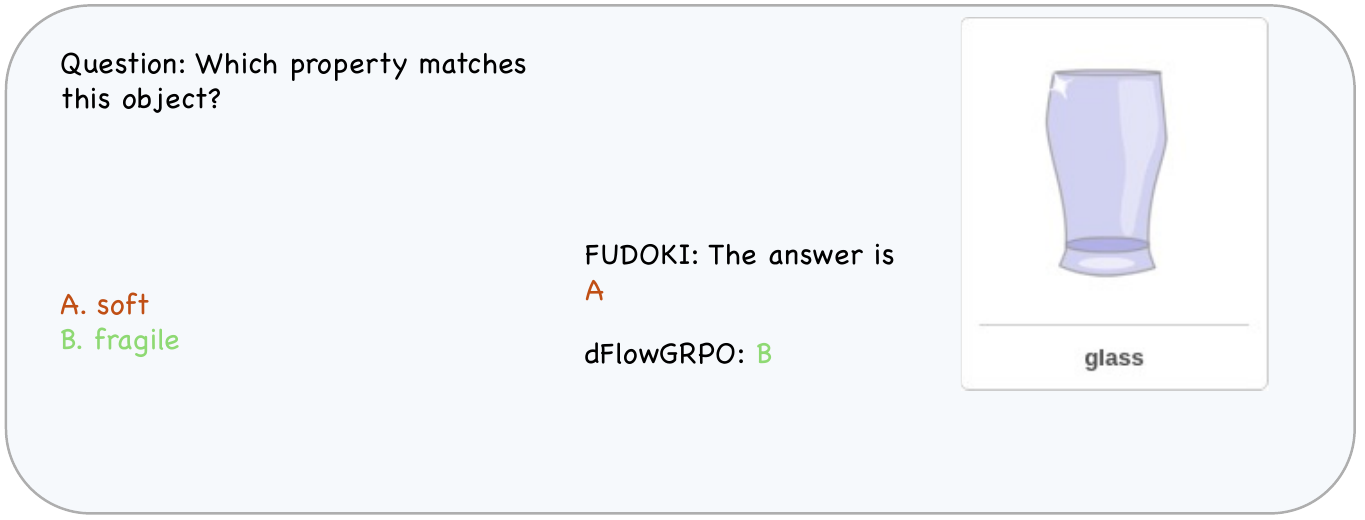}
    \caption{\small Qualitative comparison of understanding capability between dFlowGRPO and the base model FUDOKI. The correct answer is highlighted in green and the wrong answer is highlighted in red.}
    \label{fig:understanding quality4}
\end{figure}

\end{document}